\documentclass[sigconf,balance=false]{acmart}
\usepackage{popets}
\setcopyright{popets}
\copyrightyear{YYYY}
\acmYear{YYYY}
\acmVolume{YYYY}
\acmNumber{X}
\acmDOI{XXXXXXX.XXXXXXX}
\acmISBN{}
\acmConference{Proceedings on Privacy Enhancing Technologies}
\usepackage{mathtools}
\usepackage{amsthm}
\usepackage{bm}
\usepackage[noabbrev,capitalise]{cleveref} %
\usepackage{thm-restate}
\usepackage{subfiles} %
\usepackage[inline]{enumitem}
\usepackage{nicematrix}

\AtEndPreamble{%
    \theoremstyle{acmplain}
    \newtheorem{remark}[theorem]{Remark}
    \newtheorem{assumption}[theorem]{Assumption}
}

\usepackage{comment}

\crefname{assumption}{Assumption}{Assumptions}
\Crefname{ALC@unique}{Line}{Lines}

\newenvironment{proofsketch}{\renewcommand{\proofname}{Sketch of proof}\proof}{\endproof}

\usepackage{nicefrac}
\usepackage{pifont}%
\newcommand{\cmark}{\ding{51}}%
\newcommand{\xmark}{\ding{55}}%

\DeclareMathOperator*{\argmin}{argmin} %

\DeclareMathOperator{\Var}{\mathbb{V}}

\newcommand{\expect}[1]{\mathop{{}\mathbb{E}}\left[{#1}\right]}
\newcommand{\condexpect}[2]{\mathbb{E}_{#1}\left[{#2}\right]}

\providecommand{\iprod}[2]{\ensuremath{\left\langle #1,\,#2  \right\rangle}}

\newcommand{\norminf}[1]{\left\lVert{#1}\right\rVert_{\infty}}
\newcommand{\norm}[1]{\left\lVert{#1}\right\rVert}

\newcommand{\indexvar}[3]{{#3}^{\ifthenelse{\equal{#1}{}}{}{\left({#1}\right)}}_{#2}}

\newcommand{\loss}[1]{F_{#1}}
\newcommand{\modelp}[2]{x'_{#1}}

\newcommand{\gradient}[2]{\indexvar{#2}{#1}{g}}
\newcommand{\model}[2]{\indexvar{#2}{#1}{x}}

\newcommand{\avgmodel}[1]{\indexvar{#1}{}{\bar{x}}}

\newcommand{\localloss}[1]{\indexvar{}{#1}{f}}

\newcommand{\datapoint}[2]{\indexvar{#2}{#1}{\xi}}
\newcommand{\datapointp}[2]{\indexvar{#2}{#1}{\dot{\xi}}}

\def\R{\mathbb{R}}

\newcommand{\gossmatrix}[2]{\indexvar{#2}{#1}{W}}
\newcommand{\nodedegree}[1]{d_{#1}}
\newcommand{\nodeset}[0]{\mathcal{V}}
\newcommand{\neighbors}[2]{\indexvar{#2}{#1}{\Gamma}}

\newcommand{\sampledloss}[0]{f}

\providecommand{\1}{\bm{1}}
\newcommand{\transpose}[1]{\prescript{\intercal}{}{#1}}
\newcommand{\datadistrib}[1]{\mathcal{D}_{#1}}
\newcommand{\inputspace}[0]{\Omega}
\newcommand{\modelspace}[0]{\mathbb{R}^\modelsize}
\newcommand{\optimum}[0]{x^*}
\newcommand{\gradientnoise}[1]{\omega_{#1}^2}
\newcommand{\gradientnoiseavg}[0]{\bar{\omega}^2}
\newcommand{\gradientnorm}[1]{\vartheta^2_{#1}}
\newcommand{\gradientnormavg}[0]{\bar{\vartheta}^2}
\newcommand{\lr}[0]{\gamma}
\newcommand{\consdistance}[1]{\Xi_{#1}}
\newcommand{\modelmatrix}[2]{\indexvar{#2}{#1}{X}}

\newcommand{\set}[1]{\{#1\}}
\newcommand{\pview}[3]{\mathcal{O}_{#1}(\indexvar{#3}{}{\mathcal{A}}(\mathcal{#2}))}
\newcommand{\renyi}[2]{D_\alpha\left(#1\|#2\right)}

\newcommand{\privacyzsum}[3]{\indexvar{#3}{}{\epsilon}\left(#1,#2\right)}
\newcommand{\privacyzsumsingleround}[3]{\indexvar{#3}{}{g}(#1,#2)}

\newcommand{\adjacentdatasetgradientbound}[0]{\Delta}

\newcommand{\gossmessage}[3]{\indexvar{#3}{#1\to #2}{m}}
\newcommand{\sys}{\textsc{Zip-DL}\xspace}
\newcommand{\ynoise}[3]{\indexvar{#3}{#1\to #2}{Y}}
\newcommand{\ynoisestd}[1]{\varsigma_{#1}}
\newcommand{\privacyparameter}[1]{\ynoisestd{#1}}
\newcommand{\ynoisevar}[1]{\ynoisestd{#1}^2}
\newcommand{\zsumnoise}[3]{\indexvar{#3}{#1\to #2}{Z}}

\newcommand{\zsumstdmulti}[3]{\indexvar{#3}{#1\to #2}{\sigma}}
\newcommand{\zsumvarmulti}[3]{(\zsumstdmulti{#1}{#2}{#3})^2}

\newcommand{\modelsize}[0]{d}%
\usepackage[abbreviations]{foreign}
\usepackage{booktabs}

\newcommand{\covarmatrix}[2]{\indexvar{#2}{#1}{\Sigma}}
\newcommand{\modelsvirtual}[2]{\indexvar{#2}{#1}{\hat{X}}}
\newcommand{\modelspvirtual}[2]{\indexvar{#2}{#1}{\hat{\dot{X}}}}
\newcommand{\nodesetvirtual}[0]{\hat{\nodeset}}
\newcommand{\gossmatrixvirtual}[2]{\indexvar{#2}{#1}{\hat{\gossmatrix{}{}}}}
\newcommand{\mixingmatrix}[0]{\hat{M}}
\newcommand{\zsumnoisevirtual}[2]{\indexvar{#2}{#1}{\hat{Z}}}
\newcommand{\ynoisevirtual}[2]{\indexvar{#2}{#1}{\hat{Y}}}

\newcommand{\neighborsvirtual}[2]{\indexvar{#2}{#1}{\hat{\neighbors{}{}}}}
\newcommand{\zsumlinearcorrelation}[2]{\indexvar{#2}{#1}{\hat{C}}}

\newcommand{\unnoisedmodelvirtual}[2]{\indexvar{#2}{#1}{\hat{\chi}}}
\newcommand{\unnoisedmodelpvirtual}[2]{\indexvar{#2}{#1}{\hat{\dot{\chi}}}}

\newcommand{\zsumnoisetemporal}[2]{\indexvar{#2}{#1}{\tilde{Z}}}
\newcommand{\ynoisetemporal}[2]{\indexvar{#2}{#1}{\tilde{Y}}}
\newcommand{\gossmatrixtemporal}[2]{\indexvar{#2}{#1}{\tilde{W}}}
\newcommand{\zsumlinearcorrelationtemporal}[2]{\indexvar{#2}{#1}{\tilde{C}}}

\colorlet{colorsys}{red}
\colorlet{colormuffliato}{blue}
\colorlet{colornonoise}{violet}

\usepackage{comment}

\definecolor{ForestGreen}{rgb}{0.13, 0.55, 0.13}

\usepackage{etoolbox}
\usepackage{marginnote}

\usepackage{tikz}
\usepackage{pgfplots}
\usepackage{xstring}
\usepackage[eulergreek]{sansmath}
\usepackage{siunitx}
\usepackage{mdframed}

\pgfplotsset{compat=newest}
\usepgfplotslibrary{external,units,colorbrewer,groupplots,fillbetween}
\tikzsetexternalprefix{figures/}
\tikzset{external/mode=list and make}
\usetikzlibrary{patterns,fit,shapes.arrows,shapes.misc,positioning}

\usepackage{etoolbox}
\usepackage{xparse}
\usetikzlibrary{external}

\makeatletter
\begingroup\endlinechar=-1\relax
\everyeof{\noexpand}%
\edef\x{\endgroup\def\noexpand\homepath{%
        \@@input|"kpsewhich --var-value=HOME" }}\x
\makeatother

\newcommand{\inputplot}[2]{%
\includegraphics{main-figure#2.pdf}
}

\newcommand{\newgroupwidth}[2]%
{\expandafter\xdef\csname groupwidth#1\endcsname{#2}}

\newcounter{groupwidth}
\newsavebox{\groupwidthbox}
\makeatletter
{\edef\groupnumber{#1}%
	\stepcounter{groupwidth}%
	\@ifundefined{groupwidth\thegroupwidth}{\pgfmathsetlengthmacro{\mywidth}{\linewidth/\groupnumber}}%
	{\expandafter\let\expandafter\mywidth\csname groupwidth\thegroupwidth\endcsname}%
	\begin{lrbox}{\groupwidthbox}%
		\tikzset{/pgfplots/width={\mywidth}}%
		\ignorespaces}%
	{\end{lrbox}%
	\usebox\groupwidthbox
	\pgfmathsetlengthmacro{\mywidth}{\mywidth + (\linewidth - \wd\groupwidthbox)/\groupnumber}
	\immediate\write\@auxout{\string\newgroupwidth{\thegroupwidth}{\mywidth}}}
\makeatother
\usepackage{./acronym}
\acrodef{DL}{decentralized learning}
\acrodef{ML}{machine learning}
\acrodef{D-PSGD}{decentralized parallel stochastic gradient descent}
\acrodef{FL}{federated learning}
\acrodef{SGD}{stochastic gradient descent}
\acrodef{IID}{independent and identically distributed}
\acrodef{non-IID}{non independent-and-identically-distributed}
\acrodef{RMSE}{root mean square error}
\acrodef{PNDP}{Pairwise-Network Differential Privacy}
\acrodef{DP}{Differential Privacy}
\acrodef{AUC}{Area Under the Curve}
\acrodef{MIA}{Membership-Inference Attack}
\acrodef{MIAs}{Membership-Inference Attacks}
\acrodef{PNDP}{Pairwise Network Differential Privacy}
\acrodef{GN}{Group Normalization}
\acrodef{TPR}{True Positive Rate}
\acrodef{FPR}{False Positive Rate}
\acrodef{FCN}{Fully Connected Network} %

\newcommand{\cifar}{CIFAR-10\xspace}
\newcommand{\movielens}{MovieLens\xspace}
\newcommand{\sgd}{{\xspace}\ac{SGD}\xspace}
\newcommand{\dpsgd}{{\xspace}\ac{D-PSGD}\xspace}

\newcommand{\niid}{\ac{non-IID}\xspace}

\newcommand{\gn}{{\xspace}\ac{GN}\xspace} \graphicspath{ {figures/} }

\usepackage{subcaption} %

\newboolean{showcomments}
\newboolean{showrevision}
\newboolean{showsimplifiedproofs}
\newboolean{singlecolumnproof}

\setboolean{showcomments}{true}
\setboolean{showrevision}{true}
\setboolean{showsimplifiedproofs}{false} 	%

\setboolean{singlecolumnproof}{true} 	%

\ifthenelse{\boolean{showcomments}}
{%
    \newcommand{\annote}[3]{{\color{#3}%
                \colorbox{#3}{\bfseries\sffamily\tiny\textcolor{white}{#2}}
                \color{#3}
                $\blacktriangleright${\em #1}$\blacktriangleleft$}%
    }%
}{%
    \newcommand{\annote}[3]{}%
}%

\newcommand{\simplifiedproof}[2]{
    \ifthenelse{\boolean{showsimplifiedproofs}}{{#1}}{{#2}}
}

\newcommand{\singlecolumnproof}[2]{
    \ifthenelse{\boolean{singlecolumnproof}}{{#1}}{{#2}}
}

\newcommand{\todo}[1]{\annote{#1}{TODO}{blue}}

\newrobustcmd{\tagbox}[2]{\colorbox{#1}{\bfseries\sffamily\footnotesize\textcolor{white}{#2}}}
\newrobustcmd{\annoteMargin}[3]{{\hspace{-.5mm}\marginnote{\tagbox{#3}{#2}}
            \color{#3}{#1}%
        }}

\newrobustcmd{\annoteMarginRevision}[3]{{\hspace{-.5mm}\marginnote{\tagbox{#3}{#2}}
        \color{#3}{#1}%
    }}

\ifthenelse{\boolean{showcomments}}{}{%
    \renewcommand{\annoteMargin}[3]{#1}
}%
\ifthenelse{\boolean{showrevision}}{}{%
    \renewcommand{\annoteMarginRevision}[3]{#1}
}%

\usepackage{algorithm}
\usepackage{algpseudocode}

\begin{document}

\title{Low-Cost Privacy-Preserving Decentralized Learning}

\author{Sayan Biswas}
\affiliation{%
	\institution{EPFL, Switzerland}
	\city{}
	\state{}
	\country{}
}

\author{Davide Frey}
\affiliation{%
	\institution{Univ Rennes, Inria, CNRS, IRISA, France}
	\city{}
	\state{}
	\country{}
}

\author{Romaric Gaudel}
\affiliation{%
	\institution{Univ Rennes, Inria, CNRS, IRISA, France}
	\city{}
	\state{}
	\country{}
}

\author{Anne-Marie Kermarrec}
\affiliation{%
	\institution{EPFL, Switzerland}
	\city{}
	\state{}
	\country{}
}

\author{Dimitri Lerévérend}
\affiliation{%
	\institution{Univ Rennes, Inria, CNRS, IRISA, France}
	\city{}
	\state{}
	\country{}
}
\authornote{Corresponding author: \texttt{first.last}@inria.fr}

\author{Rafael Pires}
\affiliation{%
	\institution{EPFL, Switzerland}
	\city{}
	\state{}
	\country{}
}

\author{Rishi Sharma}
\affiliation{%
	\institution{EPFL, Switzerland}
	\city{}
	\state{}
	\country{}
}

\author{François Taïani}
\affiliation{%
	\institution{Univ Rennes, Inria, CNRS, IRISA, France}
	\city{}
	\state{}
	\country{}
}
\renewcommand{\shortauthors}{Biswas et al.}

\begin{abstract}
  \Ac{DL} is an emerging paradigm of collaborative
  machine learning that enables nodes in a network to train models
  collectively without sharing their raw data or relying on a central
  server. This paper introduces \sys, a privacy-aware \ac{DL} algorithm
  that leverages correlated noise to achieve robust privacy against
  local adversaries while ensuring efficient convergence at low
  communication costs.  By progressively neutralizing the noise added
  during distributed averaging, \sys combines strong privacy
  guarantees with high model accuracy.  Its design requires only one
  communication round per gradient descent iteration, significantly
  reducing communication overhead compared to competitors. 
  We establish theoretical bounds on both convergence speed and privacy guarantees. Moreover,
  extensive experiments demonstrating \sys's practical applicability
  make it outperform state-of-the-art methods in the accuracy
  vs. vulnerability trade-off. Specifically, \sys{}
\begin{enumerate*}[label=\emph{(\roman*)}] 
	\item reduces membership-inference attack success rates by up to 35\% compared to baseline \ac{DL}, 
	\item decreases attack efficacy by up to 13\% compared to competitors offering similar utility, and 
	\item achieves up to 59\% higher accuracy to completely nullify a basic attack scenario, compared to a state-of-the-art privacy-preserving approach under the same threat model. 
\end{enumerate*}
These results position \sys as a practical and efficient solution for privacy-preserving decentralized learning in real-world applications.

\end{abstract}


\maketitle

\acresetall
\section{Introduction}

\Ac{DL} allows a collection of devices %
to train a global model collaboratively without sharing raw training data.
This approach has drawn increasing attention from both academia~\cite{beltran2022decentralized} and industry, showcasing its potential across various sectors, including healthcare~\cite{lu2020decentralized, tian2023robust} and autonomous vehicles~\cite{chen2021bdfl}.
In \ac{DL}, each device (henceforth \emph{node})
\begin{enumerate*}[label=\emph{(\roman*)}]
  \item trains a local model using its own data;
  \item exchanges this model with those of its neighbors according to the underlying communication topology; and
  \item averages its current local model with the models received from neighbors.
\end{enumerate*}
This iterative process repeats until convergence is reached~\cite{ormandiGossipLearningLinear2013,lian2018asynchronous}.
Although training data never leaves participating nodes in \ac{DL}, the models that nodes exchange still leak information.
Exploiting these leaks, an honest-but-curious attacker can mount privacy attacks against participants to reveal sensitive attributes of their data.
For instance, an attacker can mount a \ac{MIA}~\cite{shokriMembershipInferenceAttacks2017,carliniMembershipInferenceAttacks2022}
that can reveal whether a particular sample belongs to the training set of a node.

\begin{table*}
  \caption{
    Position of our work compared to previous approaches. %
  }\label{tab:comparison}
  \begin{center}
    \begin{tabular}{| c | c | c | c | c |}
      \toprule
      \textsc{Approach}                            &
      \textsc{Masking (RSS-NB)}                    &
      \textsc{RSS-LB}                              &
      \textsc{Muffliato}                           &
      \textbf{\sys}
      \\
                                                   &
      \cite{gadePrivateOptimizationNetworks2018}   &
      \cite{gadePrivateOptimizationNetworks2018}   &
      \cite{cyffersMuffliatoPeertoPeerPrivacy2022} &
      \textbf{(ours)}
      \\
      \midrule
      Formal privacy guarantees                    &
      {\color{ForestGreen} \cmark}                 &
      {\color{red} \xmark}                         &
      {\color{ForestGreen} \cmark}                 &
      {\color{ForestGreen} \cmark}
      \\
      No P2P coordination                          &
      {\color{red} \xmark}                         &
      {\color{ForestGreen} \cmark}                 &
      {\color{ForestGreen} \cmark}                 &
      {\color{ForestGreen} \cmark}
      \\
      One averaging round                          & %
      {\color{ForestGreen} \cmark}                 & %
      {\color{ForestGreen} \cmark}                 & %
      {\color{red} \xmark}                         & %
      {\color{ForestGreen} \cmark}%
      \\
      Communication cost                           & %
      {\color{orange} Moderate}                    & %
      {\color{ForestGreen} Low}                    & %
      {\color{red} High}                           & %
      {\color{ForestGreen} Low}   %
      \\
      Impact on Convergence Rate                      & %
      {\color{ForestGreen} None}                   & %
      {\color{red} High}                           & %
      {\color{red} High}                           & %
      {\color{ForestGreen} Low}%
      \\
      \bottomrule
    \end{tabular}
  \end{center}
\end{table*}

\emph{\ac{DP}}~\cite{dwork:2006:calibrating} is a widely-used measure of formal privacy guarantees that has been applied to the design of privacy-preserving \ac{DL}~\cite{sabaterAccurateScalableVerifiable2022}.
\ac{DP} strategically adds noise to data so that the inclusion or exclusion of a data point becomes much harder to detect.
However, \ac{DP} typically assumes a worst-case threat model in which an attacker can access all messages transiting on the network.
As a result, although it provides robust privacy guarantees, \ac{DP} tends to require high noise levels that disrupt the learning process and severely impair the system's utility.

Following existing literature~\cite{cyffersMuffliatoPeertoPeerPrivacy2022,gadePrivateOptimizationNetworks2018}, we assume a representative threat model in which local honest-but-curious attackers can only observe the messages they receive. An attack is furthermore considered successful only if the obtained information can be linked to its contributing participant. This model covers a wide range of scenarios in which network communication is protected. Still, nodes participating in the distributed learning process can exploit their partial knowledge of the system to breach the privacy of other participants.
To specifically address this threat model, Muffliato~\cite{cyffersMuffliatoPeertoPeerPrivacy2022} introduces \emph{\ac{PNDP}}.
In contrast to \ac{DP} that captures a global privacy measure, \ac{PNDP} tracks privacy loss at a finer level, between pairs of nodes. 
As a result, \ac{PNDP} lends itself to lower noise levels, faster convergence, and better accuracy. 
Unfortunately, its use so far requires multiple rounds of averaging~\cite{cyffersMuffliatoPeertoPeerPrivacy2022}, leading to high network costs.

This paper explores the use of correlated noise to achieve \ac{PNDP} without significant network costs.
Correlated noise---a natural evolution of noise-based privacy methods---protects individual node inputs while minimizing the impact on model accuracy.
Although systems using correlated noise show promising convergence~\cite{gadePrivateOptimizationNetworks2018}, their privacy implications remain underexplored.
Several approaches using correlated noise have been formulated~\cite{sabaterAccurateScalableVerifiable2022,imtiazCorrelatedNoiseAssistedDecentralized2021,allouahPrivacyPowerCorrelated2024},
but most of them either rely on a trusted aggregator to cancel out the noises~\cite{imtiazCorrelatedNoiseAssistedDecentralized2021,sabaterAccurateScalableVerifiable2022}
or on pairwise coordination between nodes, which comes at a cost either in communication or in utility~\cite{allouahPrivacyPowerCorrelated2024}.

We introduce \sys (\emph{Zero-summing Interference for Privacy-preserving Decentralized Learning}),
a privacy-preserving algorithm that leverages correlated noise in a single communication round while offering formal privacy guarantees.
To the best of our knowledge, \sys (see \Cref{tab:comparison}) is the only approach
\begin{enumerate*}[label=\emph{(\roman*)}]
  \item with formal guarantees that
  \item requires no prior pairwise coordination between nodes, and
  \item only requires a single averaging round per gradient step.
\end{enumerate*}
In addition to \sys, we make the following contributions:
\begin{itemize}
  \item We prove that our approach converges while relying on a single communication round per gradient step.
    This powerful property results from the fact that the sum of the noise added to the communication round is zero.
    Moreover, our analysis shows that the impact of the noise on the convergence rate is negligible compared to state-of-the-art methods.
  \item We provide a formal privacy guarantee of our approach in terms of \acf{PNDP}.
\item We conduct an extensive evaluation study comparing \sys with both a state-of-the-art baseline and standard DL under threshold-based membership inference attacks (\ac{MIA}) on both the \cifar and \movielens datasets.
    Our results show that \sys provides the best trade-off between accuracy and privacy while maintaining low communication overhead. 
    In particular, \sys reduces the success rate of \ac{MIA} by up to 26 percentage points while only entailing a loss of 11 percentage points in test accuracy against baseline \ac{DL}. 
    \sys also improves test accuracy by up to 59\% w.r.t. to the state-of-the-art privacy-preserving baseline of Muffliato when configured to completely nullify a baseline threshold-based attack scenario.
  
\end{itemize}

The paper is organized as follows:
\Cref{sec:model} provides the necessary background and threat model.
\Cref{sec:zsum algorithm and properties} presents the design of \sys and its core properties.
\Cref{sec:convergence rate,sec:PNDP} present the theoretical guarantees of our privacy-preserving algorithm, in terms of convergence rate and privacy respectively.
We present the results of our experimental study in~\Cref{sec:evaluation} before surveying related work in~\Cref{sec:related works} and concluding in \Cref{sec:conclusion}.

\section{Preliminaries}\label{sec:model}
We start by describing the background and threat model considered in our work.
Sections~\ref{subsec:decentralized_opti} and~\ref{Gossip:avg}
introduce respectively general notations and the gossip-based averaging
algorithm used by most \ac{DL} algorithms.  \Cref{subsec:Privacy considerations}
describes some privacy attacks on \ac{DL} and some existing
countermeasures.  Finally, \cref{subsec:threat model} describes our
threat model.

\subsection{Decentralized learning}\label{subsec:decentralized_opti}

We consider a set of $n$ nodes $\nodeset = [\![1,n]\!]$, each owning a model of $\modelsize$ parameters, whose aim is to solve a \ac{DL} problem without sharing raw training data.
While each node, $a \in \nodeset$, stores a local data distribution $\datadistrib{a}$, the goal is to determine the model parameters $\optimum\in \modelspace$ that optimize the learning problem over the local datasets of all participating nodes.
This is done by minimizing an average loss function:
\begin{equation}\label{eqn:DL_obj}
\argmin_{\model{}{} \in \modelspace} \left[\sampledloss(\model{}{}) =\frac{1}{n} \sum_{a = 1}^n \underbrace{\condexpect{\datapoint{}{} \sim \datadistrib{a}}{\loss{a}(\model{}{}; \datapoint{}{})}}_{\localloss{a}(\model{}{})}\right]\hspace{-0.6ex},
\end{equation}
where $f_a(x)$ represents the local objective function associated with node $a$, 
and $F_a(\model{}{}; \datapoint{}{})$ quantifies the prediction loss associated with the model parameters, $\model{}{}$, for the sample, $\datapoint{}{}$, 
potentially encompassing non-convex characteristics.

To solve \cref{eqn:DL_obj}, we proceed in $T$ successive iterations, with each node, $a$, keeping its own local model, $\model{a}{t}$, for each iteration, $t\in[\![0,T]\!]$. The goal is to make the averaged model, $\avgmodel{t} := \frac{1}{n}\sum_{a=1}^{n}\model{a}{t}$, converge to $\optimum$.

The learning process involves collaborative interactions between nodes, which are connected by an underlying communication topology.
At each iteration $t$, each node first trains its model on its local data and then aims to average it with the models of other nodes.
During the averaging step, each node restricts its communication to its neighbors in the communication topology using gossip averaging (Section~\ref{Gossip:avg}).
Yet, sharing only model parameters may still leak sensitive information, thus hurting privacy.

\subsection{Gossip averaging}%
\label{Gossip:avg}

Many \ac{DL} algorithms rely on gossip averaging to estimate and share the average model $\avgmodel{t} := \frac{1}{n}\sum_{a=1}^{n}\model{a}{t}$ at each iteration $t$~\cite{devosEpidemicLearningBoosting2023,barsRefinedConvergenceTopology2023}.
A gossip-averaging operation can consist of multiple successive
rounds. In each averaging round, $s$, the nodes communicate according
to a \emph{gossip matrix}, $\gossmatrix{}{t,s}$, in the following
manner. Each node, $a$, sends a message,
$\gossmessage{a}{v}{t,s} \in \modelspace$, containing the model parameters to each neighbor, $v$. 
Upon receiving it, node $v$ weighs the received model using $\gossmatrix{a,v}{t,s}$ to perform averaging.
In the simplest setting, $\gossmessage{a}{v}{t,s}$ corresponds to the current local estimate of
$\avgmodel{t}$, this estimate is updated to
$\sum_{v\in \nodeset}\gossmatrix{a,v}{t,s}\gossmessage{v}{a}{t,s}$,
and it converges to $\avgmodel{t}$ as $s$ tends to infinity.  We make
the following assumption on $\gossmatrix{}{t,s}$:
\begin{assumption}\label{ass:mixing_matrices}
  All gossip matrices are \emph{symmetric},
  $\transpose{\gossmatrix{}{t,s}} = \gossmatrix{}{t,s}$ and \emph{stochastic},
  $\forall a\in \nodeset, \sum_{v\in \nodeset}\gossmatrix{a,v}{t,s} = 1$.
\end{assumption}

While the symmetry assumption is not always necessary~\cite{devosEpidemicLearningBoosting2023,barsRefinedConvergenceTopology2023}, it is a common assumption for complexity proofs that enables tighter bounds~\cite{koloskovaUnifiedTheoryDecentralized2020,cyffersMuffliatoPeertoPeerPrivacy2022}. 
In our case, it enables convergence and privacy analysis.

We also denote by $\neighbors{a}{t,s}$ the set of neighbors to which node $a$ sends its model, and by $\nodedegree{a}^{(t,s)}$ the corresponding degree of $a$. Formally, we have 
$   
\neighbors{a}{t,s} := \set{v\in \nodeset\mid\gossmatrix{v,a}{t,s} \neq 0}
$.
Note that, due to \cref{ass:mixing_matrices} the networks are symmetric: $v\in\neighbors{a}{t,s} \iff a\in \neighbors{v}{t,s}$.
Moreover, we consider $\neighbors{a}{t,s}$ to be a \emph{closed neighborhood} unless otherwise specified, i.e. $a\in\neighbors{a}{t,s}$. A node may thus send virtual messages to itself. This property will be pivotal to our approach and is common for efficient averaging schemes~\cite{xiaoDistributedAverageConsensus2006}.

Finally, several averaging approaches add a mask~\cite{bonawitz2017practical} or noise~\cite{cyffersMuffliatoPeertoPeerPrivacy2022} to the messages to protect the privacy of the nodes' data.
In this paper, we focus on noise-based approaches as they require less coordination between nodes and are more resilient to collusion between attackers.

\begin{remark}
  In \ac{DL}, the averaging step does not need to reach exactly the same model at each node.
  Therefore, the rounds can be stopped before full convergence. 
  In \sys, even one round is sufficient ($s=1$).
  Thus, in the rest of the paper, we will omit $s$ in equations related to the averaging process.
\end{remark}

\subsection{Privacy in Decentralized Learning}%
\label{subsec:Privacy considerations}

\paragraph{\textbf{Privacy attacks}}
Numerous privacy attacks target  Machine-Learning systems~\cite{zhuDeepLeakageGradients2019,geipingInvertingGradientsHow2020,yinSeeGradientsImage2021,shokriMembershipInferenceAttacks2017,carliniMembershipInferenceAttacks2022}. 
Most of these focus on attacking individual models or gradients, leaving attacks that exploit multiple models, such as those shared in \ac{DL}, relatively underexplored.
While some studies address this gap~\cite{shokriMembershipInferenceAttacks2017,mriniPrivacyAttacksDecentralized2024}, they often rely on strong assumptions about the attacker's capabilities or incur significant computational costs, especially in decentralized scenarios.
To address these limitations, we adopt in this work two levels of \ac{MIA} to analyze our approach: a \emph{threshold}-based attack~\cite{shokriMembershipInferenceAttacks2017,carliniMembershipInferenceAttacks2022}, and a \emph{classifier}-based \ac{MIA}~\cite{zariEfficientPassiveMembership2021a}, which offers a computationally efficient and practical approach to evaluating privacy vulnerabilities in decentralized learning.
Given some input sample $x$, \ac{MIA}s aims to decide whether $x$ belongs to a node's training set or not. Intuitively, if an attacker can accurately infer this information, it may be able
to reconstruct part of a node's training dataset. In the following, we consider an attacker co-located with one of the participating nodes (the attacking node) that observes the models it receives in order to perform an \ac{MIA} targeting the training set of some other distant node (the victim node). (See \Cref{subsec:threat model} just below for more detail.)

\paragraph{\textbf{Differential Privacy}} Introduced in databases, \acf{DP} provides a widely used framework for protecting models from such attacks~\cite{dwork:2006:calibrating,mironovRenyiDifferentialPrivacy2017,shokriMembershipInferenceAttacks2017,dworkExposedSurveyAttacks2017,cyffersMuffliatoPeertoPeerPrivacy2022}. 
In a decentralized scenario, \ac{DP} can be instantiated in different ways. The most well-studied variants are \emph{local differential privacy} (LDP)~\cite{kasiviswanathanWhatCanWe2011} and \emph{central differential privacy}. The former assumes a local model without the existence of any trusted entity (\eg{} a server) to curate the noise that, in turn, provides LDP guarantees. However, this approach is usually detrimental for the utility. The latter only provides guarantees on a final, averaged model and relies on a trusted central server for adding the noise. It has been shown that the optimal tradeoff in both cases differs by a factor of $n$, the number of nodes~\cite{duchiMinimaxOptimalProcedures2018}.
To bridge this gap, relaxations of the strict scenario of LDP have been proposed~\cite{cyffersMuffliatoPeertoPeerPrivacy2022,allouahPrivacyPowerCorrelated2024}. These relaxations include \ac{PNDP}~\cite{cyffersMuffliatoPeertoPeerPrivacy2022}, which we consider in this work and detail in \Cref{sec:PNDP}.

\subsection{Threat model}%
\label{subsec:threat model}

We aim to protect the privacy of users data against \emph{honest-but-curious} participating nodes during training. 
This scenario is in line with related work~\cite{geipingInvertingGradientsHow2020,cyffersMuffliatoPeertoPeerPrivacy2022,biswasNoiselessPrivacyPreservingDL2025} in the domain of privacy-preserving \ac{DL}, where the attacker can observe information about a victim node during training but does not deviate from the algorithm. 
We consider the attacker to be a node (or a set of nodes) participating in the training algorithm, but this can be extended to the case where an attacker is eavesdropping on a node's communication.
The attacker's goal is to infer about the victim's data, which we quantify in terms of PNDP (see~\cref{sec:PNDP} for a formal definition). This notion of privacy in the context of \ac{DL} is driven by the observation that privacy loss is not equal between all nodes in a distributed algorithm: close neighbors in the communication topology will receive more information from a node than nodes that are further away.

With \ac{PNDP} in mind, our approach perturbs the models exchanged between nodes during training to prevent some honest-but-curious attacker from inferring precise information on a victim's node training distribution.
In this setting, the attacker (which is co-located with one of the nodes) only has a partial view of the messages exchanged within the network, and its attacking power depends on its distance from the victim node within the communication graph.
This is captured formally in \Cref{def:privacy_view} in \Cref{subsec:privacy assumptions and definitions}.
In this setting, the attacker may never be in a position to reconstruct the final average model with precision, as its own local model might be biased by its position in the graph, \niid{} data, and noise injected during the training process.
The limited information that nodes can access in this setup and the key influence of their position on the (local) model they obtain
is sufficient for most applications as the resulting local models remain valuable, even if they differ from node to node.
Note, however, that if the goal of the \ac{DL} algorithm is to produce a unified global model to use by some downstream application, one may apply central-\ac{DP} or local-\ac{DP} to global model before release to obtain \ac{DP} guarantees also for this downstream application.
Such a threat model is widespread in the literature that focuses on privacy-preserving \ac{DL} such as~\cite{cyffersMuffliatoPeertoPeerPrivacy2022,biswasNoiselessPrivacyPreservingDL2025}.

To empirically evaluate the performance of \sys compared to its baselines, we conduct experiments with two paradigms of \ac{MIA} that consider an attacker with different levels of knowledge of the victim's training set. The goal is to use a victim's message to infer whether a particular training sample was used to train the victim's model in order to demonstrate how the formal privacy guarantees provided by \sys complement with the empirically mounted \ac{MIA} to capture the essence of real-world applications of our approach.
More details are given in~\cref{sec: experimental setup}.

\begin{algorithm}[t]
	\caption{\sys-averaging for a node $a$ at time $t$.}\label{alg:ZeroSum}
	\textbf{Input}: local model $\model{a}{}$, stepsize $\lr$, privacy parameter $\privacyparameter{a}$.
	\\
	\textbf{Output}: Localized model average with correlated noise.%

	\begin{algorithmic}[1]
		\State Get the gossip weights $\gossmatrix{a}{t}$,%
		\qquad $\nodedegree{a}^{(t)}\gets |\neighbors{a}{t}|$%
		\State Draw $\ynoise{a}{v}{t}\sim \mathcal{N}(0,\lr^2\ynoisevar{a})$ for $v\in \neighbors{a}{t}$\label{line:Noise_draw}
		\State $\zsumnoise{a}{v}{t} = \ynoise{a}{v}{t} -  \frac{1}{\nodedegree{a}^{(t)} \gossmatrix{a,v}{t}}\sum_{j\in\neighbors{a}{t}}\gossmatrix{a,j}{t} \ynoise{a}{j}{t}$\label{line:Noise_normalization}
		\ForAll{$v\in \neighbors{a}{t}$}
		\State Send  $\model{a}{t} + \zsumnoise{a}{v}{t}$ to $v$\label{line:send:model:a:zs:a}
		\State Receive $\model{v}{t} + \zsumnoise{v}{a}{t}$ from $v$
		\EndFor\medskip
		\State \label{line:ZeroSum_averaging} \textbf{return} $ \sum_{v\in \neighbors{a}{t}}\gossmatrix{a,v}{t} (\model{v}{t}  + \zsumnoise{v}{a}{t})$
	\end{algorithmic}
\end{algorithm}

\section{\sys: Locally-Correlated Noise}%
\label{sec:zsum algorithm and properties}

\subsection{\sys in a nutshell}

Gossip averaging typically requires multiple averaging rounds to provide a good
estimate of the average of nodes' individual
inputs~\cite{DBLP:journals/tocs/JelasityMB05}. Unfortunately, since
averaging is required at each learning iteration, these rounds add up
to a substantial network cost.

We drastically reduce this overhead by performing a single averaging round per
learning iteration.
Without noise, the cumulative effect of one-round averaging between each
gradient-descent step is enough to ensure
convergence~\cite{DBLP:conf/aistats/ZantedeschiBT20,barsRefinedConvergenceTopology2023,devosEpidemicLearningBoosting2023}.

\sys adds noise to this process to provide \ac{PNDP} guarantees. As one-round averaging is limited to a node's neighbors, the residual noise in partially averaged models remains high, which may disrupt learning and affect utility.
We mitigate this effect by \emph{correlating} the injected noise such that it
sums to zero over each node's \emph{closed neighborhood}.
The correlation is local and eschews any coordination between neighbors.

In the following, we first detail the one-round localized averaging that lies
at the core of \sys (\cref{alg:ZeroSum}), before moving on to the resulting
decentralized SGD learning algorithm (\cref{alg:ZeroSum-sgd}). We then state
some fundamental properties of \sys's global average model in
\cref{sec:core:prop}.

\subsection{Detailed description of \sys}
\label{sec:detail:desc:zsum}

\sys's model-averaging procedure is described in \cref{alg:ZeroSum}.
It relies on a stochastic communication topology~\cite{devosEpidemicLearningBoosting2023} captured by the gossip matrix $\gossmatrix{}{t}$, where $t$ denotes the current learning iteration
(\cref{Gossip:avg}).
Node $a$ first determines its neighborhood $\neighbors{a}{t}$ and the weights $\gossmatrix{a}{t}$ that its neighbors apply.
Then, to protect its local data, a node $a$ adds noise $\zsumnoise{a}{v}{t}$ to its model $\model{a}{t}$ before sending it to each of its neighbors, $v\in\neighbors{a}{t}$.
By construction, the added noise sums to zero (Lines 2-3 of~\cref{alg:ZeroSum}) so as not to affect the computation of  the global average.
A node adapts how it protects its data by picking its own privacy parameter $\privacyparameter{a}$, which itself drives the variance $\lr^2\ynoisevar{a}$ of the injected noises.

To generate zero-summing noises in~\cref{alg:ZeroSum}, a node $a$ first generates an \emph{initial noise} $\ynoise{a}{v}{t}$ (Line 2) for each of its neighbors $v$. Those noises are then correlated to create \emph{pairwise noises} $\zsumnoise{a}{v}{t}$ (Line 3) that will directly be added to the model sent to each neighbor. Those pairwise noises are the ones observed by an attacker.

In contrast to \cite{gadePrivateOptimizationNetworks2018}, \cref{alg:ZeroSum}
uses a \emph{closed neighborhood} that includes the local node $a$ (\ie $a\in
	\neighbors{a}{}$). Hence, even if $a$ is surrounded by attackers after an
eclipse attack~\cite{singhEclipseAttacksOverlay2006}, $a$'s model remains
protected to some extent as the noises of the models sent to
$\neighbors{a}{}\setminus\{a\}$ do not cancel out.

\sys's main algorithm (\cref{alg:ZeroSum-sgd}) is a \ac{DL} algorithm. At each iteration $t$, each node $a$ first performs a local gradient step on its local model $\model{a}{t}$ to produce an intermediate model $\model{a}{t+ \nicefrac{1}{2}}$ (Lines 2-3). The local model for the next iteration, $\model{a}{t+1}$, is then obtained by applying \sys's averaging procedure (\cref{alg:ZeroSum}) to this model $\model{a}{t+ \nicefrac{1}{2}}$.

\begin{algorithm}[t]
	\caption{\sys for a node $a$.}\label{alg:ZeroSum-sgd}
	\textbf{Input} $\model{a}{0}$ the initial model, $T$ the number of iterations.

	\begin{algorithmic}[1]
		\For{$t$ = 0 to $T-1$}
		\State Draw $\datapoint{a}{t}\sim \datadistrib{a}$, compute $\gradient{a}{t} := \nabla\loss{a}(\model{a}{t},\datapoint{a}{t})  $
		\State $\model{a}{t+ \nicefrac{1}{2}} = \model{a}{t} - \lr \gradient{a}{t}$
		\State $\model{a}{t+1} = \sys\text{-averaging}(\model{a}{t+\nicefrac{1}{2}}, \lr,  \ynoisestd{a})$
		\EndFor
	\end{algorithmic}
\end{algorithm}

\subsection{\sys's core properties}
\label{sec:core:prop}
The following results pave the way for the formal analysis of \sys in \cref{sec:convergence rate}.
If there is no influence of the time factor, we remove the $(t)$ superindex to alleviate the notation (e.g.\ when a lemma holds for all $t\in[\![0,T]\!]$).
Proofs that are not provided in this section can be found in~\cref{sec:zsum small properties proofs}.

First, we state a property that summarizes the effect of the noise generated by
a node on the network:

\begin{restatable}{lemma}{noiseCancellation}
	\label{lem:noise cancellation}
	\emph{Noise cancellation on the global model.} For every node $a\in\nodeset=[\![1,n]\!]$, it holds that
	\begin{align*}
		\sum_{v = 1}^n \gossmatrix{a,v}{} \zsumnoise{a}{v}{} 
        = 0
        = \sum_{v = 1}^n \gossmatrix{v,a}{} \zsumnoise{a}{v}{}.
	\end{align*}
\end{restatable}

This lemma states that a node does not add noise to the overall network, and
leads to the following crucial corollary.

\begin{restatable}{corollary}{conservationAvgModel}\label{lem:conservation average model}
	\emph{Impact on the global average model}. For every epoch $t\in [\![0,T]\!]$, we have:
	\begin{align*}
		\avgmodel{t+1} = \avgmodel{t+\nicefrac{1}{2}}.
	\end{align*}
\end{restatable}

While simple, this corollary is pivotal in our convergence analysis of
$\avgmodel{t}$. Without this property, the bound on the expectation of
$\norm{\avgmodel{t+1} - \optimum}^2$ suffers from an extra term because of the noise.

Finally, \cref{lem:corrected zsum variance} describes the behavior of the pairwise noise
generated by \sys: it follows a Gaussian distribution, which is standard for
deriving formal privacy guarantees.

\begin{restatable}{lemma}{correctedVarianceZsum}\label{lem:corrected zsum variance}
	\emph{Noise characterization for~\cref{alg:ZeroSum}.} Consider that for node $a$, for all $v\in\neighbors{a}{t}$,
	$\ynoise{a}{v}{t}\sim \mathcal{N}\left(0, \lr^2\ynoisevar{a}\right)$, for a fixed topology $\gossmatrix{}{t}$.
	Then, using the definition of~\cref{alg:ZeroSum}, we have:
	\begin{align*}
		\forall a,v\in[\![1,n]\!], \zsumnoise{a}{v}{t} \sim \mathcal{N}\left(0, \zsumvarmulti{a}{v}{t}\right)
	\end{align*}
	with
	\begin{align*}
		\zsumvarmulti{a}{v}{t} = \left(\frac{(\nodedegree{a} - 1) ^2}{\nodedegree{a}^2} +\frac{\sum_{j\in\neighbors{a}{t}, j\neq v}(\gossmatrix{a,j}{t})^2 }{(\nodedegree{a}\gossmatrix{a,v}{t})^2}\right)\lr^2\ynoisevar{a}.
	\end{align*}
\end{restatable}

Note that \cref{lem:corrected zsum variance} entails that the variance of the
noise added to sent messages is strongly linked to the communication topology.
This means that the chosen communication topology influences the privacy of our system, which further motivates our use of \ac{PNDP}~\cref{subsec:privacy assumptions and definitions}.

\begin{remark}%
	\label{rem:s-regular topology}
	When considering an $k$-regular topology or even a topology where only the incoming degree is fixed at $k$ for all the nodes with a uniform weight distribution~\cite{devosEpidemicLearningBoosting2023}, then for a node $a$, we have $\zsumvarmulti{a}{v}{t} = \frac{k-1}{k}\lr^2\ynoisevar{a}$.
    If we fix the same privacy parameter $\privacyparameter{a}$ for all nodes,
    the noises generated by individual nodes all follow the same distribution.
\end{remark}

\section{Convergence of \sys{}}\label{sec:convergence rate}

We now analyze the convergence rate of \sys{}.
The proof of the results stated in this section follows a similar structure to that of~\cite{koloskovaUnifiedTheoryDecentralized2020}.
Detailed versions of proofs related to this section can be found in~\cref{sec:convergence proofs}.

\Cref{subsec:assumptions convergence proof} describes the assumptions we used for the convergence proof. Then,~\cref{subsec:convergence rate}  details our bound in the setting described.

\subsection{Assumptions}%
\label{subsec:assumptions}%
\label{subsec:assumptions convergence proof}

To ensure convergence, we define some assumptions that are common in the literature and that mostly follow those of~\cite{koloskovaUnifiedTheoryDecentralized2020}.
First, we make assumptions about the smoothness and  convexity of the loss functions:

\begin{assumption}\label{ass:smoothness}
    \emph{(L-smoothness)}. The functions $\loss{i}: \mathbb{R}^\modelsize \times \inputspace \to \mathbb{R}$ are differentiable for each $i\in\nodeset$ and $\datapoint{}{}\in \text{supp}(\datadistrib{i})$, 
    and there exists a constant $L \geq 0$ such that for each $\model{}{},\modelp{}{} \in \modelspace$ and $\datapoint{}{} \in \text{supp}(\datadistrib{i})$:
    \begin{align}
        \label{eq:smoothness}
        \norm{\nabla\loss{i}(\modelp{}{},\datapoint{}{}) - \nabla\loss{i}(\model{}{},\datapoint{}{})} \leq L\norm{\model{}{}- \modelp{}{}}.
    \end{align}     
\end{assumption}

\begin{assumption}\label{ass:mu-convexity}
    \emph{($\mu$-convexity)} Each function $\localloss{i}$ is $\mu$-convex for a constant $\mu\geq0$. For all $\model{}{},\modelp{}{} \in \modelspace$:
    \begin{align}
        \nonumber
        \localloss{i}(\model{}{}) - \localloss{i}(\modelp{}{}) + \frac{\mu}{2}\norm{\model{}{}- \modelp{}{}}_2^2
        \leq \iprod{\nabla\localloss{i}(\model{}{})}{\model{}{}-\modelp{}{}}.
    \end{align}
\end{assumption}

We also assume the noise caused by \sgd is bounded.
This is particularly important since we consider a possible \niid data distribution:

\begin{assumption}\label{ass:bounded gradient noise}
    \emph{(Bounded noise at the optimum)} 
    Consider $\optimum$ such that $\optimum := \argmin \sampledloss(\model{}{})$ 
    and define
    \begin{align} 
        \label{eq:gradientnormdefinition}
        \gradientnorm{i} := \norm{\nabla\localloss{i}(\optimum)}^2,&    &\gradientnormavg := \frac{1}{n}\sum_{i=1}^{n}\gradientnorm{i}.
    \end{align}
    In addition, define
    \begin{align} 
        \label{eq:gradientnoisedefinition}
        \gradientnoise{i} := \condexpect{\datapoint{i}{}}{\norm{\nabla\loss{i}(\optimum,\datapoint{i}{}) - \nabla \localloss{i}(\optimum)}^2_2}
    \end{align}
    and $\gradientnoiseavg := \frac{1}{n}\sum_{i=1}^{n}\gradientnoise{i}$. Then $\gradientnormavg$ and $\gradientnoiseavg$ are bounded.
\end{assumption}

    Intuitively, $\gradientnormavg$ measures the noise level and $\gradientnoiseavg$ the diversity of the locally sampled functions $\localloss{i}$.
    It is important to note that $\gradientnoiseavg$ is strongly linked to the data distribution. 
    In particular, it will tend to be larger in a \niid setting.

Finally, we state the assumption on the mixing matrix:
\begin{assumption}\label{ass:expected consensus rate}
    \emph{(Expected consensus rate)} There exists $p\in [0,1] $ such that for all matrices $X\in\R^{\modelsize\times n}$ and all iteration $t\in[\![0,T]\!]$,
    if we define $\bar{\modelmatrix{}{}}:= \frac{1}{n}X\mathbf{1}_{n\times n}$ where $\mathbf{1}_{n\times n}\in\R^{n\times n}$ is the matrix composed of ones,
    we have
    \begin{align}
        \nonumber
        \condexpect{\gossmatrix{}{t}}{\norm{\gossmatrix{}{t}\modelmatrix{}{} - \bar{\modelmatrix{}{}}}^2_F}
        \leq (1-p)\norm{\modelmatrix{}{} - \bar{\modelmatrix{}{}}}_F^2.
    \end{align}    
\end{assumption}

This assumption is standard in the decentralized consensus literature, with $p$ a value linked to the spectrum of $\expect{\transpose{\gossmatrix{}{t}}\gossmatrix{}{t}}$~\cite{boydRandomizedGossipAlgorithms2006}.

\subsection{Convergence rates of \sys}%
\label{subsec:convergence rate}

    We now state the formal convergence of \sys in the strongly convex case:

\begin{restatable}{theorem}{zsumConvergenceRate}%
    \label{thm:convergence_rate}
    \emph{Convergence rate of \sys.}
    For any number of iterations $T$, there exists a constant stepsize $\lr$ s.t. for \cref{alg:ZeroSum-sgd}, $\frac{1}{2W_T}\sum_{t=0}^{T}w_t (\expect{\sampledloss(\avgmodel{t})} - \sampledloss^* ) + \frac{\mu}{2} r_{T+1}$ is bounded by: 
    \begin{align*}
        \mathcal{O}\left( \frac{\gradientnoiseavg}{n\mu T} + \frac{LA'}{\mu^2T^2}  + \frac{r_0L}{p}\exp\left[-\frac{\mu p(T+1)}{192\sqrt{3}L}\right]\right),
    \end{align*}
    where 
    $A' = \frac{16-4p}{2(16-7p)}(\gradientnoiseavg+ \frac{18}{p}\gradientnormavg) + \frac{\modelsize}{n}\frac{16-4p}{16-7p}\sum_{a,v=1}^{n}\nodedegree{a}\frac{{(\nodedegree{v} - 1)}^2}{\nodedegree{v}}\ynoisevar{v}$, 
    $\sampledloss^* = \sampledloss(\optimum)$, $r_{t} = \expect{\norm{\avgmodel{t} - \optimum}^2}$, 
    $w_t = {(1-\frac{\mu}{2}\lr)}^{-(t+1)}$ 
    and 
    $W_T = \frac{1}{T}\sum_{t=1}^{T} w_t$.
\end{restatable}

Or, if we prefer a formulation to reach a desired accuracy:

\begin{restatable}{corollary}{zsumConvergenceRateEpsilon}\label{cor:convergence rate epsilon}
    Setting all the constants to be the same as in \cref{thm:convergence_rate}, for any target accuracy $\rho>0$, there exists a constant stepsize $\lr$ such that~\cref{alg:ZeroSum-sgd} reaches the target accuracy after at most
    \begin{align*}
        \frac{3\kappa \gradientnoiseavg}{n\mu \rho }
        +\sqrt{\frac{3\kappa LA'}{\rho\mu^2}} 
        +\frac{192\sqrt{3}L}{\mu p}\ln{\left[\frac{3\kappa r_0L}{\rho p}\right]}
    \end{align*}
    training iterations, where $\kappa$ is the constant that arises when upper bound $\mathcal{O}\left(\frac{\gradientnoiseavg}{n\mu T} + \frac{LA'}{\mu^2T^2} + \frac{r_0L}{p}\exp\left[-\frac{\mu p(T+1)}{192\sqrt{3}L}\right]\right)$ is expanded out. 
\end{restatable}

This bound is similar to the one of~\cite{koloskovaUnifiedTheoryDecentralized2020}. The first and last terms are the same, except for the constants in the logarithm, which do not influence overall convergence  since the logarithmic term is the slowest to grow. The second term however contains the additional complexity of our approach, in particular in the definition of $A'$. Our additional term is of the form $\sqrt{\frac{3\kappa L\modelsize(16-4p)}{2n(16-7p)\mu^2\rho}\sum_{a,v=1}^{n}\nodedegree{a}\frac{{(\nodedegree{v} - 1)}^2}{\nodedegree{v}}\ynoisevar{v}}$. This term is weighted by $\rho^{-\frac{1}{2}}$ and is not the one that grows fastest as $\rho$ goes to 0, proving  the limited impact of our approach  on convergence. We observe that this term contains a weighted average of the noise propagated by every node, showing the intuitive behavior of slowing down convergence if the noise $\ynoisevar{a}$ becomes too big. Interestingly, this term grows as the network size or density grows.
Indeed, the higher the degree, the more the noise injected at each iteration, and the larger the network, the longer it takes for the noise to propagate and cancel out.

We can also compare this bound to a recent noisy approach~\cite{allouahPrivacyPowerCorrelated2024}, even if their privacy setting is different from ours.
While they do not consider a strongly-convex scenario like us and assume a weaker assumption that is implied by a strongly-convex property, we observe that the noise variance appears on their leading term, in $\mathcal{O}(\frac{1}{T})$.
The analysis we performed here on an algorithm without noise cancellation would also have yielded similar results.
On the other hand,  our approach delegates the impact of the noise to the second leading term, yielding faster convergence rates.

Similarly, the bound presented in~\cite{cyffersMuffliatoPeertoPeerPrivacy2022} is also affected in its leading term by the factor $\nicefrac{\sigma^2}{T}$, where $\sigma$ denotes the \ac{DP} noise constant, as established in Theorem 10 of~\cite{cyffersMuffliatoPeertoPeerPrivacy2022}. Consequently, the same conclusion drawn in the previous paragraph can be applied in this context.

\paragraph{Relaxation of assumptions}
Following~\cite{koloskovaUnifiedTheoryDecentralized2020}, we conjecture that our proof can be generalized to the convex and the non-convex
scenarios, thus weakening \cref{ass:mu-convexity}.
In particular, the difficulty of adapting to a non-convex scenario mostly lies in the gradient descent analysis, which is only marginally modified by our approach.
We chose to keep to the strongly convex scenario because our direct baseline also made such an assumption~\cite{cyffersMuffliatoPeertoPeerPrivacy2022}.

Likewise, we conjecture it is possible to loosen~\cref{ass:expected consensus rate} by adopting the same approach as in~\cite{koloskovaUnifiedTheoryDecentralized2020}. However, we chose to stick to a more standard assumption, as it was not the main focus of this work.

\paragraph{Node dropout}
The formal analysis of convergence of \sys{} relies on the noises canceling out on average (\Cref{lem:noise cancellation}).
In practice, nodes in \ac{DL} may have intermittent availability, \ie they may join or leave the network at any time.
As a result, the injected noise in \sys{} may not always sum to zero.
However, the inherent stochasticity of the training process and the robustness of gradient-based optimization mitigate the impact of node dropouts in \sys{}.
We experimentally demonstrate the resilience of \sys{} to node dropouts in \Cref{subsec:node dropout} and discuss the possible adaptation of our convergence proof to such scenarios.

\begin{proofsketch} (\emph{\cref{thm:convergence_rate}}).
    \sloppy
    We mostly follow the proof of~\cite{koloskovaUnifiedTheoryDecentralized2020}. 
    The main challenge lies in adapting the set of lemmas to our noisy approach. 
    The mini-batch variance (\cref{prop:mini_batch variance}) is unchanged, as it only relies on hypotheses on the loss function, which are identical to ours.
    The descent lemma (\cref{lem:descent lemma}) is where~\cref{lem:conservation average model} comes into play, since canceling noises have no impact on the averaged model.
    Without noise cancellation, an additional term would have been added here, which would have propagated to the leading term of the convergence rate in $\frac{1}{T}$.
    
    Finally, the recursion for consensus distance (\cref{lem:rec consensus distance}) is modified because of the noise addition, which becomes an extra term.
    In addition to this extra term, our main recursion is slightly altered, with an additional factor to the recursive term. 
    While this additional factor prevents solving the main recursion directly, a manipulation leads to a term that can be solved, yielding the desired result.
\end{proofsketch}
This proof relies on three main lemmas detailed in~\cref{sec:convergence proofs}.
Two of them remained unchanged using \sys{}'s properties.
For the sake of completeness, we state the adapted lemma that presents an additional last term compared to state-of-the-art \ac{DL}~\cite{koloskovaUnifiedTheoryDecentralized2020}.
This term arises from noises shifting local models from the true average.

\begin{restatable}{lemma}{lemConsensusDistance}\label{lem:rec consensus distance}
    \emph{(Recursion for consensus distance)}
    Under~\cref{ass:smoothness,ass:mu-convexity,ass:bounded gradient noise,ass:expected consensus rate},
    if stepsizes $\lr\leq \frac{p}{96\sqrt{3} L}$, then for any $\beta>0$:
    \begin{align*}
        \consdistance{t} 
        \leq& \left(1+\beta\right)\left(1-\frac{7p}{16}\right)\consdistance{t-1} 
        + \lr^2(1+\beta)\left(\gradientnoiseavg+ \frac{18}{p}\gradientnormavg\right)
        \\&+  \left(1+\beta\right)\frac{36L}{p} \left(\localloss{}(\avgmodel{t-1}) - \localloss{}(\optimum)\right)
        \\&+ \lr^2 (1 + \beta^{-1})\frac{\modelsize}{n}\sum_{i=1}^{n} \nodedegree{i} \sum_{v=1}^{n} \left(\frac{{(\nodedegree{v} - 1)}^2}{\nodedegree{v}}\ynoisevar{v}\right),
    \end{align*}
    where $\consdistance{t} = \frac{1}{n}\sum_{i=1}^{n}\condexpect{t}{\norm{\model{i}{t} - \avgmodel{t}}^2}$ is the consensus distance
\end{restatable}

\section{Pairwise Network Differential Privacy}\label{sec:PNDP}
We now formalize the privacy guarantees of \sys
in terms of \emph{pairwise-network differential privacy (PNDP)}, a graph-based variant of \ac{DP} introduced in the work by Cyffers et al.~\cite{cyffersMuffliatoPeertoPeerPrivacy2022} to capture the unique threats to privacy introduced by the \ac{DL} framework.
This section establishes the formal \ac{PNDP} guarantees that \sys provides and dissects further its analytical properties.

More concretely, we first present the additional assumption and privacy definition used in the analysis (\cref{subsec:privacy assumptions and definitions}), before defining an equivalent formulation of our algorithm~(\cref{sec:matrix notations}). 
\cref{subsec:noises over time} will then exploit this formulation to express the evolution of the system, which is pivotal to our privacy analysis of \sys presented in \cref{subsec:privacy analysis zsum-dl}. We finally consider the simpler case (\cref{alg:ZeroSum} in \cref{subsec:privacy zsum-avg}), and link our results to those of~\cite{cyffersMuffliatoPeertoPeerPrivacy2022}.

\subsection{Assumptions \& definitions}%
\label{subsec:privacy assumptions and definitions}
When discussing PNDP, we use the same notations and definitions as~\cite{cyffersMuffliatoPeertoPeerPrivacy2022}.
Specifically, with $\mathcal{D}=\left(\mathcal{D}_a\right)_{a\in\nodeset}$ denoting a set of datasets across all the nodes, we call a pair of (entire) datasets $\mathcal{D}$ and $\mathcal{D}'$ \emph{adjacent}, denoted by $\mathcal{D}\sim_a\mathcal{D}'$,
if there is some node and only one node $a\in\nodeset$ for which $\mathcal{D}_a$ and $\mathcal{D}'_a$ differ.
Considering two \emph{adjacent} datasets is the first building block to express differential privacy properties.

We analyze how \sys{} guarantees \ac{PNDP} for an input dataset $\mathcal{D}$ (a given initial data distribution between the nodes).
To this purpose, we require two additional assumptions, in addition to those highlighted in~\cref{subsec:assumptions convergence proof}. 
First, we need the distance between the models trained on two adjacent datasets to be bounded, which aligns with Assumption 1 in~\cite{cyffersMuffliatoPeertoPeerPrivacy2022}.

\begin{assumption}%
    \label{ass:input_sensitivity}
    There exists some constant $\adjacentdatasetgradientbound>0$ such that for any adjacent datasets $\mathcal{D}\sim_a\mathcal{D}'$,
    we have
    \begin{align}
        \label{eq:input_sensitivity_gradient}
        \sup_{\model{}{}\in\modelspace} \sup_{\datapoint{}{},\datapointp{}{} \in \mathcal{D}\times \mathcal{D'}} 
        \norm{\nabla\loss{}(\model{}{},\datapoint{}{}) - \nabla\loss{}(\model{}{},\datapointp{}{})}^2
        \leq \adjacentdatasetgradientbound^2.
    \end{align}

\end{assumption}
This is a standard assumption when considering differentially private algorithms: we use a bound on the original perturbation and observe how this perturbation can be scaled by the algorithm.
This assumption is typically enforced through \emph{gradient clipping}~\cite{abadiDeepLearningDifferential2016}. Due to space constraints, an analysis of its impact on \sys{} is deferred to~\cref{subsec:clipping}.

For a pair of adjacent datasets, Muffliato~\cite{cyffersMuffliatoPeertoPeerPrivacy2022} introduces the notion of \emph{privacy view} on two such datasets:
\begin{definition}\label{def:privacy_view}%
    \cite{cyffersMuffliatoPeertoPeerPrivacy2022}
    The \emph{privacy view} of a node $v$ after $T$ steps for a dataset $\mathcal{D}$ is:
    \begin{align*}
        \pview{v}{D}{T} = \set{ \gossmessage{w}{v}{t} \mid t\in[\![1,T]\!], v\in\neighbors{w}{t}} \cup \set{\model{v}{}},
    \end{align*}
    with $\mathcal{A}^{(T)}$ a state-sharing algorithm iterated $T$ times such as~\cref{alg:ZeroSum} or \cref{alg:ZeroSum-sgd}, and  $\mathcal{A}^{(T)}(\mathcal{D})$ the set of all messages sent by neighboring nodes on the network during the execution of the algorithm.
\end{definition}
The privacy view represents a projection from the set of all the messages in an execution $\mathcal{A}^{(T)}$ to the set of messages that $v$ receives during the algorithm's execution. 

When considering this privacy view $\pview{v}{D}{T}$, we consider the scenario where node $v$ would be an \emph{honest-but-curious} attacker and tries to infer information from its observations --- the privacy view.
This view is then used to define \ac{PNDP}~\cite{cyffersMuffliatoPeertoPeerPrivacy2022}, by leveraging the definition of Rényi-\ac{DP}~\cite{mironovRenyiDifferentialPrivacy2017}.

\begin{definition}\label{def:pndp}\emph{(Pairwise Network Differential Privacy)}
    For $g:\nodeset^2\to\mathbb{R}^+$ and $\alpha>1$, a mechanism $\mathcal{A}^{(T)}$ satisfies $(\alpha,g)$-\emph{Pairwise Network Differential Privacy (PNDP)} if, for all pairs of distinct nodes $a,v\in\nodeset$ and adjacent datasets $\mathcal{D}\sim_a\mathcal{D'}$, we have
    \begin{align*}
        \renyi{\pview{v}{D}{T}}{\pview{v}{D'}{T}}\leq \privacyzsumsingleround{a}{v}{T},
    \end{align*}
    where $\renyi{P}{Q}$ is the Rényi divergence~\cite{gilRenyiDivergenceMeasures2013} between probability distributions $P$ and $Q$:
    \begin{align*}
        \renyi{X}{Y} = \frac{1}{\alpha-1}\ln\int{(\frac{\mu_X(z)}{\mu_Y(z)})}^\alpha\mu_Y(z)dz,
    \end{align*}
    with $\mu_X$ and $\mu_Y$ the densities of $X$ and $Y$.
\end{definition}
Therefore, $\privacyzsumsingleround{a}{v}{T}$ quantifies the \emph{privacy leaked} from $a$ to $v$, and our goal is to constrain it to a minimal value. This decentralized approach harnesses communication topology, in contrast to \ac{DP} or Renyi-\ac{DP}, thus fully exploiting the specificity of a decentralized context. 

The choice of this privacy guarantee is further motivated by the synergy between Rényi-\ac{DP} and Gaussian noise~\cite{mironovRenyiDifferentialPrivacy2017}, as the following lemma underlines:

\begin{lemma}\label{lem:gaussian rényi bound}%
    \cite{gilRenyiDivergenceMeasures2013} Suppose that $X\sim\mathcal{N}(\mu_X, \Sigma)$ and $Y\sim\mathcal{N}(\mu_Y, \Sigma)$. Then for all $\alpha>1$, we have:
    \begin{align}
        \label{eq:gaussian rényi bound}
        \renyi{X}{Y} = \frac{\alpha}{2} \transpose{(\mu_X-\mu_Y)}\Sigma^{-1}(\mu_X-\mu_Y).
    \end{align}
\end{lemma}
This lemma is the key motivation to our use of Gaussian noises in our approach: we require an additivity property to generate cancelling noises, as well as differential privacy properties.
Thus, Gaussian noise is a natural candidate, as it fits both criterions.   

Rényi-divergence usually provides important properties when considering privacy concerns. Most notably, the \emph{composition theorem} and the preservation by \emph{post-processing}~\cite{mironovRenyiDifferentialPrivacy2017}. Of those two, the former allows for an easy way to derive the privacy guarantee of the composition of differentially private algorithms. When considering a process with multiple rounds, this makes it practical to compose privacy guarantees between rounds and significantly alleviates the analysis.

\begin{remark}
\label{rem:pndp and composition}
    Since we consider a projection of the set of all messages $\mathcal{A}^{(T)}(\mathcal{D})$ on the view of the attacker, we cannot naively apply composition theorems on $\pview{v}{D}{t}$
    to this approach directly. That is because here, the composition would rely on external information, that was not in the view of the attacker.
    To circumvent this, the original paper~\cite{cyffersMuffliatoPeertoPeerPrivacy2022} considers a full averaging algorithm, meaning composition can be performed by using the (common) final state of the averaging algorithm.
\end{remark}

However, we want a more usual view of~\ac{DL}, where we alternate between one round of averaging and one round of gradient descent.
To avoid using composition, we must be able to analyze the behavior of the noise through the gradient. 
To this end, we consider the following assumption:
\begin{assumption}%
    \label{ass:gradient gaussian perturbation}
    For all $i\in\nodeset$, for all data sample $\datapoint{i}{}$ and model $\model{}{}$, 
    if we consider a noise $Z\sim\mathcal{N}(0,\Sigma)$, then we have:
    \begin{align*}
        \nabla\loss{i}(\model{}{}+ Z, \datapoint{i}{})
        \sim \mathcal{N}(\nabla\loss{i}(\model{}{}, \datapoint{i}{}), L\Sigma).
    \end{align*}
\end{assumption}

In essence, \cref{ass:gradient gaussian perturbation} implies that the gradient of a model perturbed with Gaussian noise stays close to the unnoised (original) gradient while following a Gaussian distribution around this unnoised gradient.
The range of the standard deviation is bounded by the smoothness constant $L$ (\cref{ass:smoothness}), which comes from the remark that $\norm{  \nabla\loss{i}(\model{}{}+ Z, \datapoint{i}{}) - \nabla\loss{i}(\model{}{}, \datapoint{i}{})}^2 \leq L \norm{Z}^2 $.
This assumption will allow us to simplify privacy expressions without resorting to a composition theorem.
Most notably,~\cref{lem:unrolled distribution} links an execution of \sys to an execution of decentralized learning without any noise.
This link will be pivotal to the privacy proof.
We further evaluate~\cref{ass:gradient gaussian perturbation} in~\cref{subsec:gaussian assumption experimentation}.

\subsection{Equivalent system formulation}%
\label{sec:matrix notations}

Gossip matrices (\cref{Gossip:avg}) are a natural tool to analyze how information propagates in a communication graph over several communication rounds. Unfortunately, they cannot be directly applied to \cref{alg:ZeroSum-sgd}, as they assume that each node sends the same information to all its neighbors in a given round. This assumption does not hold for  \cref{alg:ZeroSum-sgd}, where the noise $\zsumnoise{a}{v}{}$ added by each node $a$ to its model during the \sys-averaging step (line~\ref{line:send:model:a:zs:a} of \cref{alg:ZeroSum}) is different for each of $a$'s neighbors.

We overcome this difficulty by considering an equivalent virtual communication graph of $n^2$ nodes that emulate the behavior of the $n$ nodes executing \cref{alg:ZeroSum-sgd}. In this construction, each original node $a\in\nodeset$ is replaced by $n$ virtual nodes $a_1, .., a_n\in \nodesetvirtual$ connected in a clique. Each virtual node $a_v$ is then connected to $v_a$ in the virtual communication graph if $a$ is connected to $v$.

This emulated network makes it possible to track the privacy loss incurred by our algorithm, whose behavior can be interpreted as a sequence of linear matrix operations on the states of the virtual nodes. Because each virtual node replicates the state of its real node, the system's state is encoded in a matrix of dimension $n^2 \times\modelsize$, while message exchanges and state updates are captured by matrices of size $n^2\times n^2$ (since the virtual communication topology contains $n^2$ nodes).

In the remainder of this section, we present in more detail the entities we use to analyze the privacy loss of \cref{alg:ZeroSum-sgd} using virtualization. 
Virtual entities are decorated with the symbol~$\hat{}$~: if $A$ describes an object in the original system, then $\hat{A}$ represents its counterpart in the virtual topology. 
We note $\nodesetvirtual = [\![1,n^2]\!]$ the set of virtual nodes, where the real node $i$ is represented by the virtual nodes ranging from $n(i-1)+1$ to $ni$.
$\modelsvirtual{}{t}$ represents the stacking of virtual models at time $t$, \ie,
\begin{align*}
    \modelsvirtual{}{t} = \transpose{\begin{pNiceArray}{cccccc}
        \transpose{\model{1}{t}},&
        \ldots,&
        \transpose{\model{1}{t}},&
        \transpose{\model{2}{t}},&
        \ldots,&
        \transpose{\model{n}{t}}
    \end{pNiceArray}},
\end{align*}
in which the local model $\model{a}{t}\in \mathbb{R}^{\modelsize}$ is duplicated  $n$ times across all the virtual nodes associated with node $a$. $\modelsvirtual{}{t}\in \mathbb{R}^{n^2\times\modelsize}$ in the general case, and so do the noises generated by all the nodes. 
For simplicity when defining those noises, we focus in the following on the case $\modelsize=1$ to introduce the notations, but the approach generalizes seamlessly to higher dimensions.

The noises generated in~\cref{alg:ZeroSum} are captured by two random vectors $\ynoisevirtual{}{t}$ and $\zsumnoisevirtual{}{t}$ of dimension $n^2$, defined component-wise by
\begin{align*}
    \ynoisevirtual{n(i-1)+j}{t}
    :=
    \ynoise{i}{j}{t}
    ,&&\forall i,j\in\nodeset,
    \\
    \zsumnoisevirtual{n(i-1)+j}{t}
    :=
    \zsumnoise{i}{j}{t}
    ,&&\forall i,j\in\nodeset.
\end{align*}
Due to the definition in~\cref{alg:ZeroSum}, $\zsumnoisevirtual{}{t}$ results from a linear combination of $\ynoisevirtual{}{t}$:
\begin{align}
    \label{eq:zsum noise decomposition matrix}
    \zsumnoisevirtual{}{t} = \zsumlinearcorrelation{}{t}\ynoisevirtual{}{t},
\end{align} where, $\zsumlinearcorrelation{}{t}$ is the block-diagonal matrix filled with $0$ values, except in the following positions when $a,v,j$ range over $\nodeset$:
\begin{align*}
    \zsumlinearcorrelation{n(a-1)+v,n(a-1)+j}{t} := 
    \begin{cases}
        \frac{\nodedegree{a} - 1}{\nodedegree{a}} & \text{if } j=v \land v\in\neighbors{a}{t},
        \\
        -\frac{\gossmatrix{a,j}{t}}{\nodedegree{a}\gossmatrix{a,v}{t}} & \text{if } j\neq{v} \land v\in\neighbors{a}{t},
        \\0 &\text{Otherwise}.
    \end{cases}
\end{align*}

The covariance matrix of $\ynoisevirtual{}{}$ is the diagonal matrix in which each node's variance $(\ynoisevar{a})$ is repeated $n$ times.

The covariance matrix of $\zsumnoisevirtual{}{}$ is $\covarmatrix{\zsumnoisevirtual{}{}}{t} = \zsumlinearcorrelation{}{t}\covarmatrix{\ynoisevirtual{}{}}{}\transpose{\zsumlinearcorrelation{}{t}}$ due to \cref{eq:zsum noise decomposition matrix}.

From a given gossip matrix $\gossmatrix{}{t}$, we construct $\gossmatrixvirtual{}{t}$ as the communication matrix where each virtual node only communicates with one fixed node.
We also introduce $\mixingmatrix$, which mixes information between the virtual nodes afterward. 
\begin{align}
    \nonumber
    \gossmatrixvirtual{\hat{i},\hat{j}}{t} &:= 
    \begin{cases}
        \gossmatrix{i,j}{t},    & \text{if } \hat{i} = n(i-1)+j, \hat{j} = n(j-1)+i,\\
        0,                      & \text{Otherwise}.
    \end{cases}
    \\\label{eq:mixingmatrix}
    \mixingmatrix &:= \begin{pNiceArray}{ccccc}
        \boldsymbol{1}_n&\boldsymbol{0}_n&\boldsymbol{0}_n&\dots&\boldsymbol{0}_n\\
        \boldsymbol{0}_n&\boldsymbol{1}_n&\boldsymbol{0}_n&\dots&\boldsymbol{0}_n\\
        \dots & \dots &\dots &\dots &\dots\\
        \boldsymbol{0}_n&\boldsymbol{0}_n&\boldsymbol{0}_n&\dots&\boldsymbol{1}_n\\
    \end{pNiceArray}\in\mathbb{R}^{n^2\times n^2},
\end{align}
where $\boldsymbol{1}_n=[1]_{i,j\in [1..n]}$ and $\boldsymbol{0}_n=[0]_{i,j\in [1..n]}$ represent the matrices of dimension $n\times n$ full of ones or zeros, respectively.
$\mixingmatrix$ creates a fully connected communication network between the virtual nodes of a given real node. In doing so it captures how each local node averages the individual models it receives through $\gossmatrixvirtual{}{t}$.

Using this matrix, we obtain the following virtual gossip round:
\begin{align}    
    \label{eq:matrix update rule} %
    \modelsvirtual{}{t+1} = \mixingmatrix\gossmatrixvirtual{}{t}(\modelsvirtual{}{t+\nicefrac{1}{2}} + \zsumnoisevirtual{}{t}).
\end{align}

The following lemma ensures that the update rule stays the same as \cref{line:ZeroSum_averaging} of~\cref{alg:ZeroSum}, 
proving we have constructed something equivalent to the non-virtual update rule: 
\begin{restatable}{lemma}{lemTemporalMatrixNotation}%
    \label{lem:matrix notation consistency}
    Consider $i\in \nodeset$ and $t\in\mathbb{N}$. Then we have:
    \begin{align*}
        \forall k\in\nodeset, \modelsvirtual{ni+k}{t} = \modelmatrix{i}{t}.
    \end{align*}
\end{restatable}

\subsection{Accounting for noises over time}\label{subsec:noises over time}

In order to track how privacy losses propagate from one SGD round to the next without using a composition theorem (see~\cref{rem:pndp and composition}), we further consider $T$ successive rounds of \cref{alg:ZeroSum-sgd}.
These $T$ rounds incur the generation of $Tn^2$ individual noise values at Line 2 %
of \cref{alg:ZeroSum}.
We track the correlation between these noises and the model parameters to which they are applied in the virtual system through covariance matrices of size $tn^2$, for $t\in [\![1,T]\!]$.%

To track those $n^2 \times t$ noises, we consider matrices that aggregate data through time for notation purposes. 
Those matrices will be denoted by a ~$\tilde{}$~ notation. 
Similarly to before, we consider $\ynoisetemporal{}{t}\in\mathbb{R}^{tn^2}$ a matrix stacking all the noises generated on the network.

Even if the noises at time $t+1$ are independent from the noises at time $t$, meaning the covariance matrix will be block-diagonal, we reach a simpler expression with time matrices.
Formally, we have:
\begin{align}
    \covarmatrix{\ynoisetemporal{}{t}}{} := 
    \begin{pNiceArray}{cccc}
        \covarmatrix{\hat{Y}}{} & \boldsymbol{0}_{n^2} & \dots & \boldsymbol{0}_{n^2} \\
        \boldsymbol{0}_{n^2} & \covarmatrix{\hat{Y}}{} & \dots & \boldsymbol{0}_{n^2} \\
        \dots & \dots & \dots & \dots \\
        \boldsymbol{0}_{n^2} & \boldsymbol{0}_{n^2} & \dots & \covarmatrix{\hat{Y}}{}
    \end{pNiceArray}
    \in \mathbb{R}^{tn^2 \times tn^2},
\end{align}
where $\covarmatrix{\hat{Y}}{}\in\mathbb{R}^{n\times n}$ corresponds to the covariance matrix of the uncorrelated noises. 
This is a diagonal matrix. 
In the special case where all nodes have the same privacy parameter $\privacyparameter{}$, then we have $\Sigma_{\ynoisetemporal{}{t}} = \ynoisevar{}I_{tn^2 \times tn^2}$.

Using this and the decomposition
$\zsumnoisevirtual{}{t} =\zsumlinearcorrelation{}{t}\ynoisevirtual{}{t}$ (\cref{eq:zsum noise decomposition matrix}), 
where $\ynoisevirtual{}{t}\sim\mathcal{N}(0,\Sigma_{\hat{Y}})$, 
we also create a decomposition $\zsumnoisetemporal{}{t} = \tilde{C}^{(t)}\ynoisetemporal{}{t}$, 
where $\zsumlinearcorrelationtemporal{}{t}$ a block diagonal matrix of all the $\zsumlinearcorrelation{}{t}$.

For ease of notation, when considering matrices that aggregate through time, we will consider a constant communication matrix $\gossmatrix{}{t} = \gossmatrix{}{}$. 
Our notations could be generalized at the expense of matrix product notations.
For the temporal gossip matrix, we define the following:
\begin{align}
    \label{eq:parrallel executions distribution}
    \gossmatrixtemporal{}{T}:=  \begin{pNiceArray}{c}
        {(1-\lr L)}^{T} \\
        \dots, \\ 
        (1-\lr L)
    \end{pNiceArray}
    \begin{pNiceArray}{ccc}
        {(\mixingmatrix\gossmatrixvirtual{}{})}^T, & 
        \dots, & 
        \mixingmatrix\gossmatrixvirtual{}{}
    \end{pNiceArray}.
\end{align}
In particular, we have $\gossmatrixtemporal{}{T}\in \mathbb{R}^{n^2\times Tn^2}$.
This matrix will appear in~\cref{thm:zsumSgdPrivacy} and can be used to compute the propagation of the noise through the system after $T$ steps. 

This notation finally allows us to leverage~\cref{ass:gradient gaussian perturbation}.
Using the equivalent formulation defined in~\cref{sec:matrix notations}, we now progress toward the privacy analysis. 
First, we derive the distribution of the model vectors: 
\begin{restatable}{lemma}{LemUnrolledDitribution}%
    \label{lem:unrolled distribution}
    Using~\cref{ass:gradient gaussian perturbation}, consider $\unnoisedmodelvirtual{}{T}$ a virtual execution without any noise, and every other source of randomness is the same. 
    Then, we have:
    \begin{align*}
        \modelsvirtual{}{T}
        &\sim 
        \mathcal{N}\left(
            \unnoisedmodelvirtual{}{T}
            ,L\gossmatrixvirtual{}{T}\zsumlinearcorrelationtemporal{}{t}\Sigma_{\ynoisetemporal{}{t}}\transpose{(\gossmatrixvirtual{}{T}\zsumlinearcorrelationtemporal{}{t})}
        \right).
    \end{align*}
\end{restatable}
This lemma draws a parallel between an execution of~\cref{alg:ZeroSum-sgd} and an unnoised execution and is at the core of our privacy analysis. \Cref{lem:unrolled distribution} offers a structure to bound the Rényi divergence between $\modelsvirtual{}{T}$ on two executions on adjacent datasets. Its proof is deferred to~\cref{sec:zsumsgd-pndp proof}. 

\subsection{\sys privacy analysis}
\label{subsec:privacy analysis zsum-dl}
We now focus on analyzing the formal privacy guarantees of \cref{alg:ZeroSum-sgd}.

\begin{restatable}[Privacy of \sys]{theorem}{zsumSgdPrivacy}%
\label{thm:zsumSgdPrivacy}
    $T$ iterations of \sys (\cref{alg:ZeroSum-sgd}) satisfies $\left(\alpha,\privacyzsum{a}{v}{T}\right)$-PNDP,
    where $\privacyzsum{a}{v}{T}$ is bounded for any two nodes $a,v\in\nodeset$ by:
    \begin{align*}
        \frac{2\alpha\lr^2\adjacentdatasetgradientbound^2}{L + 4\lr^2L^2}\sum_{t=0}^{T-1}\sum_{\substack{\hat{v}\in\hat{V}\\\hat{w}\in\neighborsvirtual{\hat{v}}{t}}}
        \frac{
            {(2 + 4\lr^2L)}^t -1
        }{
            {\left(
                {\left(\gossmatrixtemporal{}{}\zsumlinearcorrelationtemporal{}{}\right)}^{(t)}\tilde{\Sigma}_{\tilde{Y}^{(t)}}\transpose{\left(\gossmatrixtemporal{}{}\zsumlinearcorrelationtemporal{}{}\right)}^{(t)}
            \right)}_{\tilde{w},\tilde{w}}
        },
    \end{align*}
    where $\tilde{\Sigma}_{\tilde{Y}^{(t)}}$ is a diagonal matrix representing the noise variances of all noises generated by the algorithm up to time $T$,
    $\zsumlinearcorrelationtemporal{}{t}$ is a block-diagonal matrix representing the correlation factor at each iteration $t$, 
    and $\gossmatrixtemporal{}{t}$ is the accumulation of all the powers of the gossip matrix defined in~\cref{sec:matrix notations}.
\end{restatable}

In essence, a node's privacy loss increases over time, and the influence of the privacy mechanism is denoted by the denominator: this term accounts for all the noises received by the virtual node $\hat{w}$.
On the other hand, the numerator accounts for how models drift away from each other.

If we consider that all nodes have the same privacy parameter $\privacyparameter{}$, then the denominator becomes akin to the norm of $(\tilde{W}\tilde{M})^{(t)}_{\hat{w}}$, which is similar to~\cite{cyffersMuffliatoPeertoPeerPrivacy2022}.

This result is a double sum over time and the attacker's neighbors, since in our notation $\neighborsvirtual{\hat{v}}{t}$ is a set containing at most one value that translates whether $w$ is in $\neighbors{v}{t}$ or not.

\begin{remark}
    This result naturally extends to colluding nodes if we consider $\hat{V} = \bigcup_{v\in V} \set{n(v-1)+k \mid k\in \nodeset}$ to be the set of colluding nodes.
    We can thus have a similar bound of $\privacyzsum{a}{V}{T}$, for $V\subset \nodeset$ a set of colluding nodes.
\end{remark}

Even if the matrices considered here are of large dimensions, this bound can be computed in practice since their underlying matrices are sparse:
either they are diagonal by block, or some have only one element by line. 
For instance, both $\tilde{M}$ and $\tilde{\Sigma}_{\tilde{Y}^{(t)}}$ are diagonal by block since the noises generated at each iteration are independent.

\begin{remark}
    In practice,~\cref{ass:gradient gaussian perturbation} may not always hold accurately. To capture the ripple effect of this inaccuracy and bound the privacy loss in this scenario, one may define an error term stemming from~\cref{ass:gradient gaussian perturbation}, using Corollary 4 of~\cite{mironovRenyiDifferentialPrivacy2017}, and add the following term to~\cref{thm:zsumSgdPrivacy}:
    \[
    D_{\infty}\left( 
    \modelsvirtual{}{T} 
    \| 
    \mathcal{N}\left(
            \unnoisedmodelvirtual{}{T}
            ,L\gossmatrixvirtual{}{T}\zsumlinearcorrelationtemporal{}{t}\Sigma_{\ynoisetemporal{}{t}}\transpose{(\gossmatrixvirtual{}{T}\zsumlinearcorrelationtemporal{}{t})}
        \right)
    \right)
    \]
\end{remark}

\subsection{\sys-avg privacy analysis}%
\label{subsec:privacy zsum-avg}
We also focus on the privacy of~\cref{alg:ZeroSum} as a pure averaging algorithm. 
This removes gradient from the proof of~\cref{thm:zsumSgdPrivacy}, and thus~\cref{ass:gradient gaussian perturbation} is not needed.
By following the same proof with a simpler update rule, we can derive a more tractable term, 
\begin{restatable}{theorem}{ThmPndpZsumAvg}%
    \label{thm:zsum-avg pndp}
    $T$ iterations of~\cref{alg:ZeroSum} satisfy $(\alpha,\privacyzsum{a}{v}{T})$-\ac{PNDP},
    where $\privacyzsum{a}{v}{T}$ is bounded for any two nodes $a,v\in\nodeset$ by:
    \begin{align*}
        \frac{\alpha\Delta^2}{2}\sum_{t=0}^{T-1}\sum_{\hat{v}\in\hat{V}} \sum_{\hat{w}\in\neighborsvirtual{\hat{v}}{t}}
        \frac{
            \left((\mixingmatrix\gossmatrixvirtual{}{})^T\right)_{\hat{w},\hat{a}}
        }{
            \left(
                (\tilde{W}\zsumlinearcorrelationtemporal{}{})^{(t)} \Sigma_{\tilde{Y}} \transpose{(\tilde{W}\zsumlinearcorrelationtemporal{}{})^{(t)}} 
            \right)_{\tilde{w},\tilde{w}}
        },
    \end{align*}
    where 
    \begin{align*}
        \gossmatrixtemporal{}{T}:= 
        \begin{pNiceArray}{ccc}
            (\mixingmatrix\gossmatrixvirtual{}{})^T, &
            \dots, & 
            \mixingmatrix\gossmatrixvirtual{}{}
        \end{pNiceArray}.
    \end{align*}
\end{restatable}

\begin{remark}%
    \label{rem:full-averaging pndp}
    This theorem generalizes the result of~\cite{cyffersMuffliatoPeertoPeerPrivacy2022}
    by introducing
    the correlation matrix between all the generated noises $\zsumlinearcorrelationtemporal{}{t}$.
    Applied to the algorithm presented in~\cite{cyffersMuffliatoPeertoPeerPrivacy2022}, the correlation matrix $\zsumlinearcorrelationtemporal{}{t}$ in the above expression would instead be the identity matrix.
    Additionally, the numerator is also the same as the one of the original work, as we have $\left((\mixingmatrix\gossmatrixvirtual{}{})^T\right)_{\hat{w},\hat{a}} = (\gossmatrix{}{})_{w,a}^T$ where $w,a$ are the nodes associated to the virtual nodes $\hat{w},\hat{a}$.
\end{remark} 
\section{Evaluation}%
\label{sec:evaluation}

We evaluate \sys{} on two practical learning tasks, image classification (on CIFAR-10) and movie recommendation (on \movielens{}). We compare \sys{}'s performance\footnote{All code used in this section can be found at \url{https://github.com/dimiarbre/ZIP-DL}} to that of two baselines: \ac{D-PSGD}~\cite{lian2017dpsgd}, and Muffliato~\cite{cyffersMuffliatoPeertoPeerPrivacy2022}, a privacy-preserving \ac{DL} algorithm.
The comparison focuses on two aspects: (i) the tradeoff between privacy (measured as the ROC-AUC of two membership inference attacks) and model utility (top-1 test accuracy on CIFAR-10, and test loss on \movielens{}), and (ii) the cost of such privacy in terms of communication overhead.
For the sake of completeness, we also consider \ac{TPR} at low \ac{FPR} rates for one of these attacks, in line with existing literature~\cite{carliniMembershipInferenceAttacks2022}. Those results are reported in~\cref{subsec:TPR at low FPR}.

\subsection{Experimental setup}%
\label{sec: experimental setup}

\paragraph{Communication graph}
Throughout the evaluation, we use \num{100} nodes connected in a 6-regular communication graph.
We assume all nodes are online and available at all times unless stated otherwise.
The experiments with node dropouts are presented in \Cref{subsec:node dropout}.

\paragraph{Baselines}

We compare \sys with two baselines: \ac{D-PSGD}~\cite{lian2017dpsgd}, a decentralized stochastic gradient descent algorithm without privacy guarantees (labeled \emph{No noise} in the figures), and Muffliato, a state-of-the art privacy-preserving \ac{DL} algorithm.
To cover different communication settings, we consider two scenarios:
\textit{(i)} using $1$ averaging round per training iteration, as in typical \dpsgd{}, and 
\textit{(ii)} using $10$ averaging rounds per training iteration, as recommended for Muffliato when applied to our network topology~(Theorem 5 in~\cite{cyffersMuffliatoPeertoPeerPrivacy2022}).
In practice, using $10$ averaging rounds ensures each node obtains an almost global average model.

\paragraph{First Learning Task --- \cifar}
We evaluate \sys{} and the two baselines on the image classification task of \cifar~\cite{krizhevsky2014cifar} using a \gn{}-ResNet18~\cite{heDeepResidualLearning2016}.
We opted for \gn{} layers~\cite{wuGroupNormalization2018} over the traditional Batch Normalization layers due to their superior compatibility with differential privacy mechanisms, especially in decentralized scenarios~\cite{yuNotLetPrivacy2021,nasirigerdehKernelNormalizedConvolutional2023}.
The training set comprises \num{50000} data samples and the test set \num{10000} data samples, spread uniformly between the nodes.
The neural network has \num{11189312} trainable parameters. For utility, we consider the top-1 test accuracy for \cifar{}.

\paragraph{Second Learning Task --- \movielens{}} We consider a recommendation task on the \movielens{} dataset~\cite{harperMovieLensDatasetsHistory2016}.
We use \textit{movielens-small}\footnote{\url{https://files.grouplens.org/datasets/movielens/ml-latest-small-README.html}}, a dataset of  \num{100836} ratings from \num{610} users on \num{9742} movies, where each user has rated at least 20 movies.
Given \movielens is naturally partitioned between users, we allocate users to a given node. We use a matrix factorization model with the \ac{SGD} optimizer. While it is similar to a classification task, it provides three new aspects in our experimental evaluation: 
\begin{enumerate*}[label=\emph{(\roman*)}]
    \item the underlying model is linear,
    \item the data is naturally \niid{},
    \item the model outputs numerical rating estimates (between $1$ and $5$).
\end{enumerate*}
Because model outputs are numerical,
we use the RMSE on predicted ratings (which corresponds to the task's test loss) as the utility metric for \movielens{}.

\paragraph{Noise levels}
Both \sys{} and Muffliato use noise levels of the form $k\sigma$, with $k\in \{2^0,2^1,\ldots,2^8\}$ and $\sigma$ such that $128\sigma=0.225$.
This range of values covers a broad set of behaviors for both \sys{} and Muffliato on both tasks (CIFAR-10 and Movielens).
Some additional, intermediary values of $k$ are also used for relevant plots. 
More precisely, we directly use $k\sigma$ as the standard deviation of noise in Muffliato, while we derive a uniform privacy parameter $\privacyparameter{a}$ from $k\sigma$ for \sys{} using the formula in \cref{rem:s-regular topology}.
Doing so ensures that the standard deviation of pairwise noises $\zsumstdmulti{a}{v}{t}$ is $k\sigma$ in both approaches.

\paragraph{Privacy attacks}

We evaluate the privacy of the algorithms against an \textit{honest-but-curious} attacker described in \Cref{subsec:threat model} using two membership inference attacks (MIA).
We apply \textit{(i)} a threshold-based attack~\cite{shokriMembershipInferenceAttacks2017, carliniMembershipInferenceAttacks2022},
and \textit{(ii)} a more advanced classifier-based attack described in~\cite{zariEfficientPassiveMembership2021a} and inspired by~\cite{nasrComprehensivePrivacyAnalysis2019}.
Both attacks seek to determine whether a victim node used a specific data point to train its local \ac{ML} model.
Taking an example as input, both attacks base their decisions on the example's loss computed by the node's local model, under the assumption that lower losses are indicative of training examples. The effectiveness of these attacks is evaluated using the \ac{AUC} of the \ac{TPR} plotted as a function of the \ac{FPR}. (This curve is known as \emph{Receiver operating characteristic} or ROC, hence the shorthand ROC-AUC.)

The threshold attack reaches its decision by comparing the example's loss obtained from the victim's final model to some fixed threshold.
While simple, this approach establishes a baseline for privacy vulnerability: if such an attack proves successful, it implies that more sophisticated methods are likely to succeed as well~\cite{carliniMembershipInferenceAttacks2022}.

While the threshold attack uses a single model, the classifier attack records multiple models shared by a single victim at different time stamps during the training process. 
For a given data example, the attacker considers the time series of the loss of this example on those models.
The aim of the classifier is then to classify this time series as either a member or non-member of the victim node's local dataset. 
More precisely, the attack uses a \ac{FCN} binary classifier with 2 hidden layers and uses the losses of $26$ models obtained at fixed intervals during the victim's training process (including its first and final model).
The binary classifier is trained on $70\%$ of the victim's node local dataset (positive training examples) and $70\%$ of the test data (negative training examples). Examples are re-weighted so that both classes have the same weight in the training. The attack is then evaluated on the remaining examples (local and test) not used for training. 
Because it requires a training phase with sufficient data, this attack assumes the attacker possesses considerable knowledge about the victim's dataset and is therefore much more aggressive than the threshold attack.

\paragraph{Utility-privacy trade-off and communication cost}
In addition to each algorithm's privacy, we measure the best utility across models during training, computed on each task's test dataset (using top-1 accuracy for CIFAR-10, and RMSE on \movielens{} MovieLens). We further measure the overall communication cost of the entire training process (measured in Teribytes).
The utility and privacy measures are averaged over all nodes.

\begin{figure}[t]
	\centering
	\inputplot{plots/aadv_tradeoff_cifar10_iid_threshold}{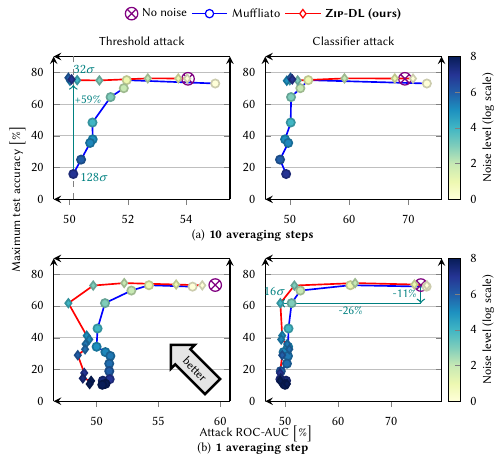}
	\caption{
		Maximum accuracy reached as a function of both attacks results on \cifar{}.
		Color intensity represents a higher noise level. In all cases, the tradeoff is better for~\sys{}.
	}%
	\label{fig:cifar_tradeoffs}
    \Description{A set of four plots, each depicting two approaches: Muffliato and \sys{} at different noise levels for the last two. In addition, a single point for each point representing No Noise is present to set a baseline. Each plot has the attack ROC-AUC as the x-axis, and the maximum test accuracy as the y-axis. The top plots represent results for 10 averaging steps, while the bottom plots for 1 averaging step. The left is for the threshold attack, while the right is for the classifier attack. In all cases, increasing noise levels make the attack ROC-AUC decrease. However, the accuracy decrease is more prominent for Muflliato, yielding to a better tradeoff for \sys{}. Additionally, for the top plots, no decrease in accuracy can be found for \sys{}, but the decrease in attack ROC-AUC is still present.}
\end{figure}

\subsection{\sys{} privacy-utility tradeoff}%
\label{subsec:tradeoff privacy-accuracy}

\begin{figure}[t]
	\centering
	\inputplot{plots/aadv_tradeoff_movielens_threshold+classifier}{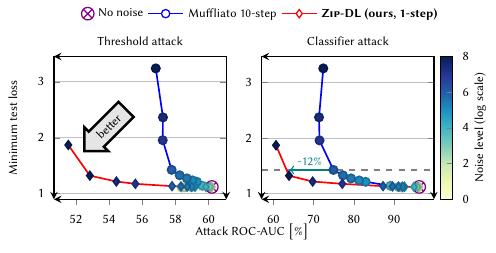}
	\caption{
		Minimum test loss as a function of attack results for varying noise levels on \movielens. Color intensity represents a higher noise level. \sys{} consistently shows better utility for equivalent attacker advantage. 
	}%
	\label{fig:movielens_tradeoffs}
    \Description{Two side-by-side plots, each depicting two approaches: Muffliato with 10 averaging steps and \sys{} with 1 averaging step, both at multiple noise levels. In addition, a single point for each plot representing No Noise is present to set a baseline. Each plot has the attack ROC-AUC as the x-axis, and the minimum test loss (RMSE Loss) as the y-axis. The left is for the threshold attack, while the right is for the classifier attack. In all cases, increasing noise levels make the attack ROC-AUC decrease. However, the RMSE increase is more prominent for Muflliato. For instance, on the right plot, the test loss suddenly increases without decreasing the attack ROC-AUC at around 75\% ROC-AUC for Muffliato, whereas it only happens at around 65\% ROC-AUC for \sys{}, yielding to a better tradeoff for \sys{}. On the right plot, there is also a line at around 1.4 test loss, with a 12\% ROC-AUC difference.
    }
\end{figure}

\Cref{fig:cifar_tradeoffs,fig:movielens_tradeoffs} show the privacy-utility trade-off of Muffliato (blue curve) and \sys (red curve) for multiple noise levels (represented by the filling color of the data points). The privacy-utility trade-off of \ac{D-PSGD} (``No noise'', i.e. no privacy protection) is shown for reference as a purple cross in a circle. \Cref{fig:cifar_tradeoffs} charts the results obtained on CIFAR-10, while \Cref{fig:movielens_tradeoffs} plots those obtained with \movielens{}.

\paragraph{CIFAR-10 Results} In \Cref{fig:cifar_tradeoffs}, the average test accuracy across nodes is shown vertically (higher is better) and the attack's success (ROC-AUC) horizontally, with lower (and better) values on the left. A $50\%$ ROC-AUC indicates the attack has been neutralized, as in this case the attacker performs similarly to a random binary classifier. Better results are in the top-left corner (higher utility for a lower attack success). \Cref{fig:cifar_tradeoffs} presents the results obtained for the two attacks in two communication settings: using 10 averaging steps (top row), and using 1 averaging step (bottom row). Results using the (basic) threshold attack are shown on the left, while those with the (stronger) classifier attack are on the right.

The charts show \sys either matches or outperforms Muffliato in all four combinations. \sys's advantage is notable at higher levels of protection when the ROC-AUC measure approaches $50\%$, where \sys provides a better protection for the same utility, or a better utility for the same protection. This is particularly visible for 10 steps averaging (top row), where \sys is able to neutralize the attack with close to no drop in accuracy. For instance, under the Threshold Attack, \sys achieves a mean ROC AUC of $50.07\%$ for an accuracy of $75.56\%$, a drop of only $0.25\%$ percentage point against \emph{No Noise} ($75.81\%$). \sys yields similar results under the Classifier Attack ($50.42\%$ ROC-AUC for the same test accuracy). 
This represents a substantial improvement over Muffliato which collapses at higher noise levels. For instance, under the Threshold Attack, Muffliato reaches a ROC-AUC of $50.13\%$ at an accuracy of only $15.95\%$, representing a drop of $59.86\%$ against \sys (shown as a vertical cyan arrow on the figure.)

\paragraph{\movielens{}} In contrast to \Cref{fig:cifar_tradeoffs}, \Cref{fig:movielens_tradeoffs} captures model utility through average RMSE across nodes (test loss, vertical axis), where lower values are better. \Cref{fig:cifar_tradeoffs} shows the results of the two attacks (Threshold Attack, left, and Classifier Attack, right) with different communication set-ups for the two competitors. While Muffliato uses 10 averaging steps (following the guidelines of~\cite{cyffersMuffliatoPeertoPeerPrivacy2022}), \sys only uses one. This allows Muffliato to converge to an almost exact average model across all 100 nodes between each learning iteration, while forcing \sys to rely on imperfect averages limited to a node's immediate neighborhood. In spite of this disadvantage, \sys clearly outperforms Muffliato on this task. Under the Threshold Attack, \sys reaches a ROC-AUC of $52.84\%$ for a test loss of $1.32$ (second point from the left), a loss increase of only $0.21$ against \emph{No Noise} (ROC-AUC $60.20\%$ for a test loss of $1.11$), while Muffliato's ROC-AUC never gets below $56.535\%$, with a loss that diverges rapidly as noise increases.
On the more powerful Classifier Attack, both \sys and Muffliato grant a lesser protection. In spite of this, \sys's protection advantage over its competitor is even higher, yielding a ROC-AUC of $63.87\%$ at a loss of $1.32\%$, an improvement of close to $12\%$ ROC-AUC points over the protection granted by Muffliato for a close to identical loss (ROC-AUC of $74.90\%$ for a loss of $1.42$, shown as an horizontal cyan arrow on the figure).
For completeness, a figure detailing all averaging steps, as outlined in~\cref{fig:cifar_tradeoffs}, can be found in~\cref{sec:additional experiments}.

\paragraph{Comparison} Comparing the results obtained on the two tasks, we observe clearly distinct behaviors:
\begin{itemize}
	\item \cifar{}: We observe in~\cref{fig:cifar_tradeoffs} that even small perturbations can significantly reduce the attack efficiency. However, a higher noise level simply reduces the model's utility. 
	\item \movielens{}: By contrast, low noise levels have minimal influence on attack performance. But higher noise levels succeed in consistently dampening the attack's efficiency.
\end{itemize}
We conjecture that these different behaviors result from the fact that  \movielens{} considers a metric space in which the values of loss have a direct meaning for the attacker, while loss values are not directly significant in \cifar{}.  This is also the reason why we consider test accuracy on \cifar{} and the test loss on \movielens{}. 
Regardless of these considerations, however, our results show that \sys{} outperforms Muffliato in both cases.

\begin{figure}[t]
	\centering
	\inputplot{plots/aadv_bars_nodyn}{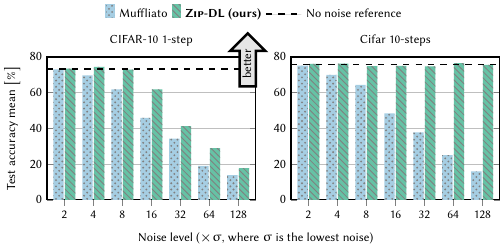}
	\caption{
		Best average test accuracy at different noise levels on \cifar. \sys{} is able to consistently reach higher accuracies, even for much higher noise levels.
	}%
	\label{fig:aadvbars_nodyn}
\end{figure}

\paragraph{Impact of the noise level} \Cref{fig:aadvbars_nodyn} compares the best accuracy reached by \sys and Muffliato for various noise levels on \cifar{}.
In contrast to Muffliato, the accuracy of \sys is less sensitive to noise in the region of high test accuracy, \textit{i.e.}, \sys with a noise level of $8\sigma$ achieves similar test accuracy to Muffliato with a much lower noise level of $2\sigma$.
Furthermore, this remains true even when comparing \sys with Muffliato 10-steps (even though Muffliato has a $10\times$ communication cost). 
Interestingly, \sys with 10 averaging steps does not experience any decrease in accuracy for all the noise levels we evaluated, since the noise cancellation can happen without proceeding through gradient descent.
In conclusion, \sys demonstrates better convergence when compared to Muffliato for similar privacy vulnerabilities, without requiring additional averaging steps.

\subsection{Communication overhead}

\begin{figure}[t]
	\centering
	\inputplot{plots/muffliato_perf_v2}{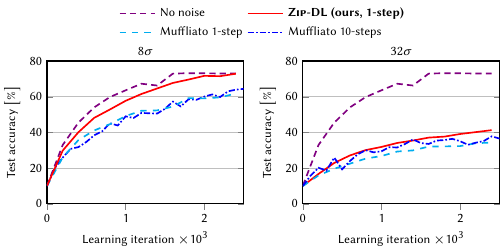}
	\caption{
		Muffliato test accuracy with different numbers of averaging rounds for a noise level of $8\sigma$ (left) and $32\sigma$
		(right), compared to \sys (1-round). Even for high noise levels ($32\sigma$), Muffliato 10-steps marginally improves the performances over Muffliato 1-step, and \sys{} manages to beat both with only one averaging round.
	}%
	\label{fig:muffliato_learning}
\end{figure}

While basic DL and \sys limit themselves to a single averaging round per gradient descent step, Muffliato is designed to perform several of them to ensure the convergence of the averaging process.
The exact number of communication steps required depends both on the variance of the models and datasets at the nodes and on the spectral analysis of the communication graph~\cite{cyffersMuffliatoPeertoPeerPrivacy2022}.

In our experimental scenario and with the same distribution assumptions as in Muffliato's original paper~\cite{cyffersMuffliatoPeertoPeerPrivacy2022}, we find that Muffliato's performs best with at least 10 averaging steps.

\Cref{fig:muffliato_learning} shows the evolution of the test accuracy w.r.t. the number of iterations for basic DL, Muffliato (with 1 and 10 averaging rounds), and \sys (with 1 averaging round) for two different noise levels. For both settings, we observe that Muffliato-10 is as accurate or more accurate than Muffliato-1.
As evidenced by \cref{fig:muffliato_learning}, these additional averaging rounds have a heavier impact on the behavior of the test loss for higher noise levels, such as $32\sigma$, but do not necessarily have much of an impact from a convergence point of view for lower noise levels, such as $8\sigma$.

Our baseline~\cite{cyffersMuffliatoPeertoPeerPrivacy2022} also considers Chebyshev polynomials for faster convergence.
However, this only partially reduces the required averaging rounds and does not affect the results of~\cref{subsec:tradeoff privacy-accuracy} that only considers privacy and utility properties.

\begin{figure}[t]
	\centering
	\inputplot{plots/databars}{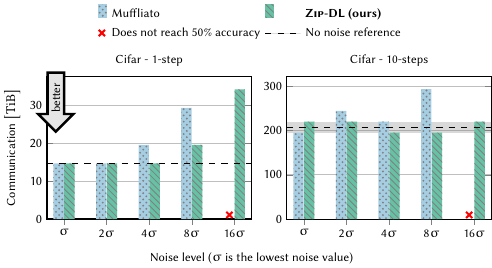}
	\caption{
		Communication cost to reach \num{50}\% accuracy for \cifar.
		Muffliato fails to reach this target for higher noise levels.
	}%
	\label{fig:databars}
\end{figure}

The addition of noise in both \sys{} and Muffliato does not affect only the final utility of the models. In some cases, this noise increases the number of learning iterations required for accuracy to converge, as the learning process must compensate for this noise, yielding more training time and communication overhead while also potentially exposing more information.
We focus on this communication overhead and measure it using the total number of bytes transferred to reach $50\%$ top-1 accuracy for both \sys{} and Muffliato.
\Cref{fig:databars} shows the communication overhead in \SI{}{\tebi\byte} for increasing noise levels on the \cifar{} dataset.
Being sensitive to the noise level, Muffliato does not even converge to an accuracy of $50\%$ for noise levels beyond $16\sigma$.
\sys{}, therefore provides better privacy guarantees while limiting the communication overhead.

Finally, we also observe in~\cref{fig:databars} that performing 10 averaging steps has a significant overhead when considering sizeable models, especially in combination with small local datasets that require frequent communication rounds between gradient descents. Here, running with 10 averaging steps yields an order of magnitude more communication cost to reach similar accuracy.

\begin{figure}[t]
	\centering
	\inputplot{plots/databars_movielens}{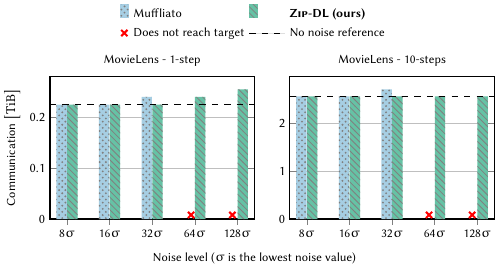}
	\caption{
		Communication cost to reach \num{1.25}$\times$ the best test loss of unnoised \ac{DL} for \movielens.
		Muffliato fails to reach this target for higher noise levels.
	}%
	\label{fig:databars_movielens}
\end{figure}

All those observations are reflected in~\cref{fig:databars_movielens}, which represents a comparable plot for the \movielens dataset. Unlike the previous experiments, we focus here on the communication cost required to achieve a given target test loss rather than a target test accuracy, due to the nature of the \movielens task. The target test loss is determined by adding $25\%$ to the best unnoised test loss.

\subsection{Other evaluation metrics}%
\label{subsec:TPR at low FPR}

\begin{figure}[ht!]
    \centering
    \inputplot{plots/aadv_tradeoff_classifier_tpratlowfpr}{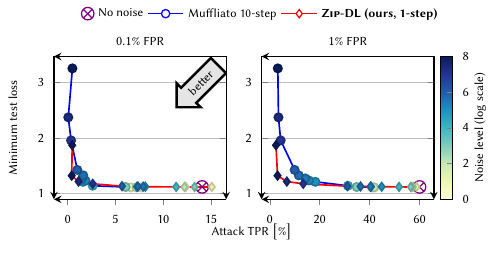}%
    \caption{\ac{TPR} at low \ac{FPR} for the \movielens{} dataset.}%
    \label{fig:movielens_tpr_at_low_fpr}
\end{figure}

For completeness, we report in~\cref{fig:movielens_tpr_at_low_fpr} the privacy/utility tradeoff of \sys and Muffliato using the attacker's \ac{TPR} at a low \ac{FPR} on the \movielens{} dataset. Intuitively, this represents the attacker's ability to identify valid training examples (\ac{TPR}) when seeking to make close to no error (low \ac{FPR}). We consider \ac{FPR} values of both $0.1\%$ and $1\%$, following~\cite{carliniMembershipInferenceAttacks2022}.
For a \ac{FPR} of $0.1\%$, the attack can be thwarted efficiently (down to a \ac{TPR} of about $2.5\%$) with close to no impact on the test loss with both approaches, rendering a fined-grained comparison difficult.
But for a \ac{FPR} of $1\%$, \sys{} provides a better tradeoff between test-loss and attack success, delivering better predictions for low \ac{TPR} values ($\text{TPR}<10\%$, bottom-left corner of the right-handed subfigure).

To further motivate our choice of attacks, we tried implementing a more recent \ac{MIA}~\cite{carliniMembershipInferenceAttacks2022}, but obtained almost worse results than random guesses in our \ac{DL} scenarios. We conjecture this is due to both the \niid{} nature of the data distribution in \ac{DL} and the fact that \ac{DL} usually considers small local training datasets. We believe these render this attack (and most shadow training attacks) impractical, and we view this as a promising avenue for future works. 
\section{Related work}%
\label{sec:related works}

\paragraph{\textbf{Correlated noises}} Correlated noises are a natural choice when seeking to reduce the utility cost of privacy. However, most of the literature focuses on correlation across nodes~\cite{sabaterAccurateScalableVerifiable2022,imtiazCorrelatedNoiseAssistedDecentralized2021,allouahPrivacyPowerCorrelated2024}.
In order to apply such correlation, participating nodes need to rely either on a trusted aggregator, so that the noise can cancel out~\cite{imtiazCorrelatedNoiseAssistedDecentralized2021,sabaterAccurateScalableVerifiable2022}, or
on an agreement between nodes~\cite{sabaterAccurateScalableVerifiable2022,allouahPrivacyPowerCorrelated2024}. We argue that the former is not always achievable, nor desirable, and the latter comes at a cost in terms of communication or  utility~\cite{allouahPrivacyPowerCorrelated2024}.

A recent approach also leveraging correlated noises, DECOR~\cite{allouahPrivacyPowerCorrelated2024}, assumes that the channels between nodes may get compromised, and considers an adversary with access to every message transmitted on the network. 
Under this strong threat model, DECOR leverages shared secrets and introduces a novel privacy criterion, \emph{secret-based local differential privacy} (SecLDP), which is orthogonal to \ac{PNDP} considered in this paper. 
In particular, SecLDP is conditioned on the number of pairwise secrets compromised by the attacker. 
To counter such a strong adversary, DECOR injects a combination of independent and correlated noises that require pairwise coordination between nodes. 
This strategy is overkill in the presence of honest-but-curious adversaries and comes at a cost in terms of convergence and accuracy, a drawback \sys does not exhibit (see \cref{subsec:convergence rate}).

\paragraph{\textbf{Other methods to achieve privacy}}
Other variants of correlated noise, such as \emph{secret sharing}~\cite{shamirHowShareSecret1979} can be used in the context of~\ac{DL}~\cite{mcmahan2017communicationefficient}.
While additive secret sharing does not necessarily involve coordination among nodes or a trusted aggregator, it requires multiple averaging rounds to reconstruct the shared average.
Thus, additional operations, such as gradient descent, cannot be mixed together with the communication process, leading to prohibitive communication costs. 
\sys{}, on the other hand is able to mix together communication and gradient computation (see \cref{tab:comparison}).

Other cryptographic approaches include \emph{secure multiparty computation}~\cite{kanagavelu2022fed} and \emph{secure aggregation}~\cite{bonawitz2017practical}.
In these techniques, nodes agree on masks that conceal local models during the averaging process.
Despite providing exact solutions to model averaging, they impose a significant drawback by requiring nodes to coordinate in order to set up and remove the masking.
In large and dynamic distributed systems, this requirement may prove infeasible, especially in real-world scenarios involving mobile devices.
For this reason, we designed our approach to avoid such coordination.

\paragraph{\textbf{Combining with other privacy mechanisms}}
Because our approach requires no additional communication between nodes, it can also be used in combination with other approaches. For instance, uncorrelated noises could also be added to our approach.
This would allow to have an intermediary approach, with possibly stronger privacy protection. Combining correlated and uncorrelated noises was proved to be possible~\cite{allouahPrivacyPowerCorrelated2024} for other privacy definitions, the main difference with our approach being how the correlation is performed. Importantly, our approach does not necessitate any sort of coordination, making it friendly to combine with other works.
However, this work focuses on fully correlated noises and their impact, and the study of such combinations is left for future work. 
\section{Conclusion}
\label{sec:conclusion}

\Ac{DL} makes a step towards privacy in collaborative learning by preventing raw data sharing. 
However, models shared between nodes still leak private information.
We introduce \sys, which enhances privacy in \ac{DL} by injecting correlated noise into shared models.
\sys does not introduce additional messages or require any sort of coordination across nodes, hence having minimal impact on communication cost while keeping convergence rates on par with the state-of-the-art.
In particular, the noise introduced by \sys has a provably lower impact on the convergence rate of the system than other similar approaches. \sys can thus be used as a basic privacy addition even in high-performance regimes where traditional privacy-preserving mechanisms may be unusable because of utility degradation.
We prove formal privacy guarantees for \sys in terms of \ac{PNDP}, bounding the privacy leakage of a node.
Experimental results confirm \sys's superior privacy-accuracy tradeoff under two paradigms of membership inference attacks with different levels of underlying strengths of the adversary's knowledge. 
\sys performs particularly well on attacks that do not require crossing information across iterations, which are the most studied practical attack scenarios.
Future work will explore broader scenarios beyond the initial assumptions of symmetric gossip matrices and behavior of a noisy gradient, aiming to extend \sys's applicability and robustness guarantees. 

\begin{acks}
    Co-authors affiliated to EPFL have  been funded by the Swiss National Science Foundation, under the project `FRIDAY: Frugal, Privacy-Aware and Practical Decentralized Learning', SNSF proposal No. 10.001.796.
    
    Experiments presented in this paper were carried out using the Grid'5000 testbed, supported by a scientific interest group hosted by Inria and including CNRS, RENATER and several Universities as well as other organizations (see \href{https://www.grid5000.fr}{https://www.grid5000.fr}).

    This work was performed using HPC resources from GENCI-IDRIS (Grant 2024-AD011015352)
\end{acks}

\bibliographystyle{ACM-Reference-Format}
\bibliography{references}

\appendix
\crefalias{section}{appendix}

\begin{table*}[ht]%
    \label{tab:symbol-table}
    \caption{List of the main symbols used in this work.%
    }
    \centering
    \begin{tabular}{|c|l|}
        \hline
        Symbol &  Usage 
        \\\hline $\nodeset$ & Set of all the nodes that participate in the training.
        \\\hline $n$ & Number of nodes in $\nodeset$.
        \\\hline $a,u,v $ & Nodes in $\nodeset$.
        \\\hline $\neighbors{a}{t,s}$ & Neighbors of node $a$ at averaging round $s$, after learning iteration $t$.
        \\\hline $\nodedegree{a}^{(t,s)} $ & Degree of node $a$ at averaging round $s$, after learning iteration $t$.
        \\\hline $\nodedegree{a} $ & Maximum degree of node $a$, over learning iterations and averaging rounds.
        \\\hline $\gossmatrix{}{t,s} $ & Gossip matrix at averaging round $s$, after learning iteration $t$.
        \\\hline $ p$ & Mixing parameter of the gossip matrices (\cref{ass:mixing_matrices}).
        \\\hline $\model{a}{t}$ & Model of node $a$ at learning iteration $t$.
        \\\hline $\avgmodel{t} $ & Average model at learning iteration $t$.
        \\\hline $\model{a}{t+\nicefrac{1}{2}}$ & Model of node $a$ at learning iteration $t$ after the gradient step.
        \\\hline $\avgmodel{t+\nicefrac{1}{2}} $ & Average model at learning iteration $t$ after the gradient step.
        \\\hline $\optimum$ & Optimal model.
        \\\hline $\sampledloss^*$ & Minimum of the global loss function.
        \\\hline $\datadistrib{a} $ & Data distribution of node $a$.
        \\\hline $ \datapoint{a}{t}$ & Data sample drawn from $\datadistrib{a}$.
        \\\hline $\loss{a}$ & Loss function of node $a$.
        \\\hline $\localloss{a} $ & Sampled (or expected) loss of node $a$ (\cref{eqn:DL_obj}).
        \\\hline $\sampledloss$ & Globally sampled loss (\cref{eqn:DL_obj}).
        \\\hline $\mu $ & Convexity constant (\cref{ass:smoothness}).
        \\\hline $L$ & Smoothness constant (\cref{ass:smoothness}).
        \\\hline $\gradientnorm{i}$ & Noise level at the optimum (\cref{ass:bounded gradient noise}).
        \\\hline $\gradientnoise{i}$ & Diversity of the data distribution at the optimum (\cref{ass:bounded gradient noise}).
        \\\hline $\lr $ & Stepsize of the gradient descent.
        \\\hline $ \ynoise{a}{v}{t}$ & Intermediate noise generated by node $a$ destined to $v$ at learning iteration $t$.
        \\\hline $ \zsumnoise{a}{v}{t}$ & \sys-averaging noise from node $a$ to node $v$ at learning iteration $t$.
        \\\hline $ \ynoisevar{a}$ & Variance of $\ynoise{a}{v}{t}$.
        \\\hline $ \zsumvarmulti{a}{v}{t}$ & Variance of $\zsumnoise{a}{v}{t}$.
        \\\hline $ \Delta $ & Adjacent datasets bound (\cref{ass:input_sensitivity}).
        \\\hline $\privacyzsumsingleround{a}{v}{T}$ & Privacy bound from node $a$ to node $v$ at timestamp T (\cref{def:pndp}). 
        \\\hline $ \modelsvirtual{}{t} $ & Virtual models vector at time $t$. 
        \\\hline $ \unnoisedmodelvirtual{}{t}$ & Unnoised virtual execution (with the same graphs and batches, but no noise) at time $t$. 
        \\\hline $\mixingmatrix$ & (Virtual) Mixing matrix (\cref{eq:mixingmatrix})
        \\\hline $\transpose{A}$ & Transpose of some matrix $A$.
        \\\hline
    \end{tabular}
\end{table*} 

\section{Additional experiments details}\label{sec:additional experiments}

\subsection{\cifar: \sys 1---step vs Muffliato 10---steps}
\begin{figure}[ht!]
    \vskip 0.2in
    \centering
    \inputplot{plots/aadv_tradeoff_cifar10_iid_10steps_threshold}{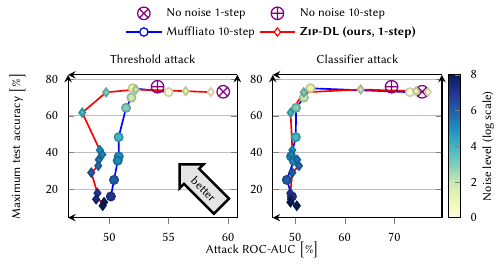}%
    \caption{The privacy-utility tradeoff of \sys{} compared to Muffliato 10-steps, on the \cifar dataset. Even though \sys reach a similar tradeoff to Muffliato 10-steps, \sys requires fewer communications.}%
    \label{fig:cifar_10steps_vs_1steps_tradeoff}
\end{figure}
As was done in~\cref{subsec:tradeoff privacy-accuracy} for \movielens{}, we also directly compare \sys{} to Muffliato with 10 averaging steps on \cifar{} in~\cref{fig:cifar_10steps_vs_1steps_tradeoff}. 
Looking at threshold attack, \sys{} offers a better tradeoff. On the other hand, the tradeoffs of both approaches are very similar for the classifier attack. But closing this gap that is present in~\cref{fig:cifar_tradeoffs} comes at the cost of $10$ averaging steps, thus inducing a significant communication overhead. Thus, \sys{} can achieve similar privacy properties compared to Muffliato 10-steps but reduces the communication overhead by a factor of $10$.  

\subsection{\movielens: \sys 1---step vs Muffliato 10---steps}
\begin{figure}[ht!]
    \centering
    \inputplot{plots/aadv_tradeoff_movielens_allattacks_allsteps}{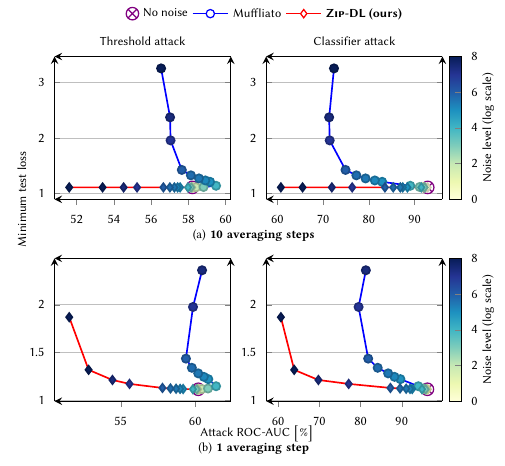}%
    \caption{The privacy-utility tradeoff of \sys{} compared to Muffliato on \movielens, for multiple averaging steps.}%
    \label{fig:movielens_allattacks_allsteps}
\end{figure}

We also display the full tradeoffs for \movielens in~\cref{fig:movielens_allattacks_allsteps}, mimicking what was done for \cifar{} in~\cref{fig:cifar_tradeoffs}.
We observe a similar tendency: our approach systematically offers a better tradeoff.
Moreover, the tradeoff offered by Muffliato 1--steps is significantly worse compared to Muffliato 10--steps.
This further motivates our choice for~\cref{fig:movielens_tradeoffs} to directly compare \sys{} with 1 averaging step versus Muffliato with 10 averaging steps. 
\subsection{Clipping}%
\label{subsec:clipping}

\begin{figure}[ht!]
    \centering
    \inputplot{plots/movielens_clipping_convergence}{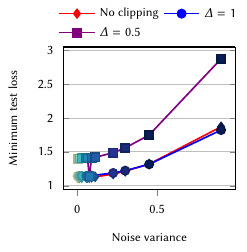}%
    \caption{Best test loss of \sys for different clipping parameters, on the \movielens dataset.}%
    \label{fig:movielens_clipping}
\end{figure}

In practice, \cref{ass:input_sensitivity} can be enforced through gradient clipping, a standard approach in \ac{DP}-\ac{SGD}~\cite{allouahPrivacyPowerCorrelated2024, abadiDeepLearningDifferential2016}. However, gradient clipping introduces additional constraints on theoretical convergence, and in some scenarios, it may even prevent convergence entirely~\cite{zhangUnderstandingClippingFederated2022}.

To assess the impact of clipping on the convergence guarantees of \sys{}, we evaluate its effect on convergence in~\cref{fig:movielens_clipping}. Our experiments focus on the \movielens{} dataset with two clipping parameters: $\Delta = 1$ and $\Delta = 0.5$.

\cref{fig:movielens_clipping} shows that the convergence of \sys{} is only marginally affected by the clipping parameter, even when the noise variance approaches the gradient bound. This illustrates that \cref{ass:input_sensitivity} can be effectively enforced in practice through gradient clipping without significantly compromising convergence.

\subsection{Experiments with varying number of nodes}

\begin{figure}[ht!]
    \centering
    \inputplot{plots/aadv_tradeoff_movielens_classifier_32_64_nodes}{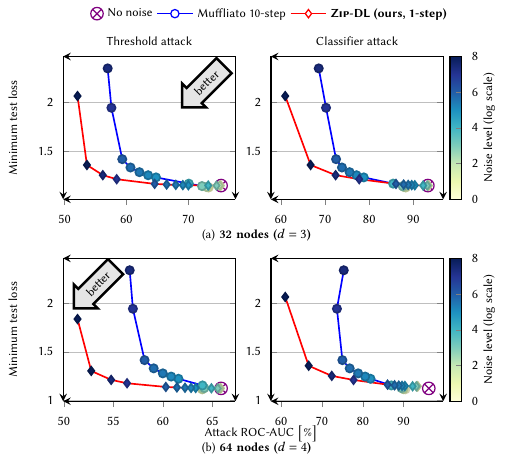}%
    \caption{Varying number of nodes for the \movielens{} dataset.}%
    \label{fig:movielens_multiple_nodescount}
\end{figure}

For completeness, we also explore varying the number of nodes partaking in \ac{DL}, which was fixed to $100$ in the main paper.
We vary this count to respectively $32$ and $64$ nodes in~\cref{fig:movielens_multiple_nodescount}.

We observe the tradeoffs keep the same tendency as the ones of~\cref{fig:movielens_tradeoffs}.
Moreover, the order of magnitude of this tradeoff is consistent, even if the number of nodes and the degree of the graph was changed.

\subsection{Node dropout}%
\label{subsec:node dropout}

\begin{figure}[ht!]
    \centering
    \inputplot{plots/dropout}{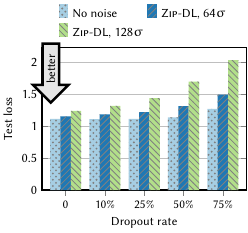}%
    \caption{Test loss for a fixed communication budget with varying dropout rates for the \movielens{} dataset.}%
    \label{fig:movielens_dropout}
\end{figure}

Since \sys{} relies on correlated noise sent in different directions through the network, it is interesting to evaluate how node dropout impacts \sys{}'s performances. Such a scenario is not computationally expensive: since each node behaves independently, a node dropout does not introduce any need for restarting a round. This is one of the benefits of our approach compared to secret sharing since we add noise and not masks.
However, node dropout will naturally impact any \ac{DL} approach. In our case, it may change~\cref{lem:noise cancellation}, since the propagated noise will not be exactly zero. We evaluate in this section how node dropout affects \sys{}.

We call node dropout when a node skips a communication round. Such a node does not compute any gradient and does not receive or send any message. However, it may come back online later in the training.
To do so, we simulate four levels of dropout: $10\%$, $25\%$, $50\%$ and $75\%$. We also add a dropout correlation of $10\%$, making the dropped out nodes more likely to remain dropped out.

\cref{fig:movielens_dropout} reports the result, by considering the test loss for two noise levels and those for dropout rates. 
We observe that dropout makes higher noise level deteriorate more in test loss. 
However, even for high levels of dropout and noise, the degradation remains marginal as long as the dropout rate is below $25\%$. 

This decrease in utility for high levels of dropout is to be expected, as this is outside the considerations of our theoretical convergence guarantee. In particular, node dropout makes~\cref{lem:noise cancellation} not hold, which was pivotal in our proof. However, we conjecture that our convergence proof could be adapted to capture node dropouts, at the cost of a term in the form of $\nicefrac{\sigma^2}{T}$. This would mean our convergence mostly matches the existing literature~\cite{cyffersMuffliatoPeertoPeerPrivacy2022,allouahPrivacyPowerCorrelated2024} in terms of noise impact. However, how the dropout rate will affect the noise term in $\nicefrac{1}{T}$ remains a future work, as our approach is initially designed for networks with a low amount of dropout to leverage the noise cancellation.  

\subsection{Empirical evaluation of \cref{ass:gradient gaussian perturbation}}%
\label{subsec:gaussian assumption experimentation}

\begin{figure}
    \centering
    \inputplot{plots/hist_sigma.tex}{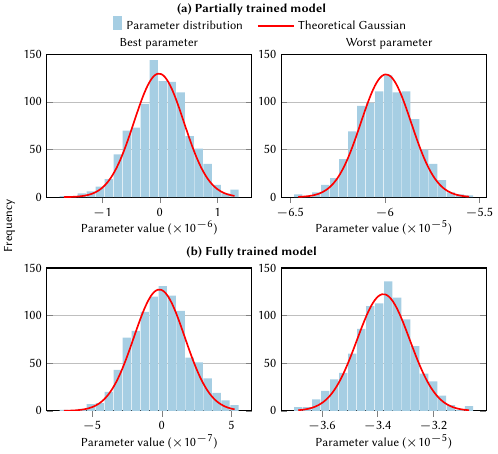}
    \caption{
        Distribution of the noisy gradient for two parameters on the \movielens{} dataset, with low noise level $\sigma$
    }%
    \label{fig:noisy_gradient_lownoise}
\end{figure}

\begin{figure}
    \centering
    \inputplot{plots/hist_128sigma.tex}{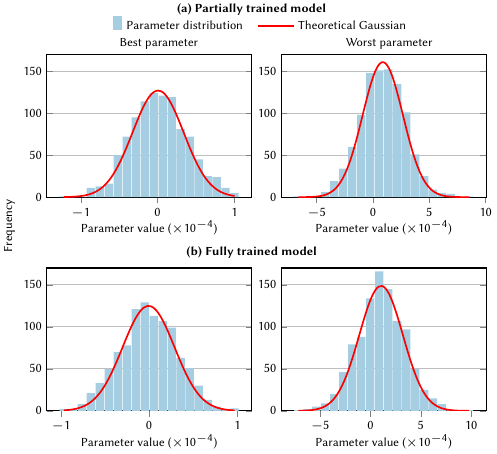}
    \caption{
        Distribution of the noisy gradient for two parameters on the \movielens{} dataset, with high noise level $128\sigma$
    }%
    \label{fig:noisy_gradient_highnoise}
\end{figure}

We empirically examine some simple scenarios to provide motivation for and, in turn, justify~\cref{ass:gradient gaussian perturbation}.

To achieve this, we utilize a centralized \movielens{} task with the same parameters as those described in~\cref{sec: experimental setup}. We consider two cases: a fully trained model (after $100$ iterations) and a partially trained model (after $10$ iterations). For each case, we examine two data samples of size $1000$ each. In the first sample, we compute the true gradient and subsequently generate a noisy gradient by perturbing this true gradient. In the second sample, we first add noise to the model parameters and then compute the gradient. We evaluate the differences between the two normalized distributions by calculating the Kolmogorov-Smirnov test statistic for all parameters, given that the exact Lipschitz constant $L$ is unknown. 

For clarity, we present two histograms illustrating the distribution in the scenario where we first add noise and then compute the gradient: one for the best-case and the other for the worst-case Kolmogorov-Smirnov test outcomes. To optimize computational efficiency, we randomly select $10,000$ parameters for calculating the aforementioned metrics.

\cref{fig:noisy_gradient_lownoise} shows the results for a low noise level ($\sigma$), while~\cref{fig:noisy_gradient_highnoise} illustrates the outcomes for a higher noise level ($128\sigma$).
More specifically,~\cref{fig:noisy_gradient_highnoise,fig:noisy_gradient_lownoise}, subplots (a), presents the case with a partially trained model, whereas~\cref{fig:noisy_gradient_highnoise,fig:noisy_gradient_lownoise}, subplots (b), represents gradients computed on a noise around the fully trained model.
Across all scenarios, a consistent trend is observed for all parameters, indicating that the parameters generally follow a normal distribution.

\subsection{Additional attacks details}

In this section, we provide further details to the \ac{MIA} workflow for both the threshold and the classifier attack. We also specify what information the attacker has access to.

\paragraph{Attacker observations during training}
During training, an attacker stores models received from its neighbors at different iterations, while denoting whose model is being saved. It is important to note those are noisy models. For \cifar, the attacker stores one every $100$ iterations, whereas it is one every $50$ iteration for \movielens.

\paragraph{Threshold attack}
The attacker computes the losses generated by the model on both the victim's local training set and the global test set. The attacker evaluates the ROC-\ac{AUC} for each saved model.
This yields an attack result for each logged model, meaning we can also observe tendencies across iterations.

\paragraph{Classifier attack}
After training, the attacker groups all the models received that were sent by a target victim node $v$. For a given data point $x$, the attacker computes the loss of $x$ through all logged models of $v$. This creates a time series of losses for this data point $x$.
We can label those time series considering whether $x\in\datadistrib{v}$ or not.
The attacker then trains a classifier to discriminate between time series, using a train-test split between both the testing set and the victim's local training set. $70\%$ of the local training dataset and the testing set are used to train the classifier, and the remaining $30\%$ of both sets are used for evaluation. Reweighting is performed to account for the unbalance in the class distributions. 
\section{Privacy proof}\label{sec:zsumsgd-pndp proof}
In this section, we provide the necessary steps to prove \cref{thm:zsumSgdPrivacy}.
Most necessary assumptions are detailed in~\cref{subsec:privacy assumptions and definitions,subsec:assumptions convergence proof}, but 
\Cref{subsec:privacy proof assumptions} details common technical assumptions in the field, or express intermediary results necessary for the main result.
\simplifiedproof{}{Then, \Cref{subsec:privacy proof main theorem} proves the main privacy theorem.}%
Finally, \cref{subsec:zsumsgd-pndp proof} contains details about \cref{subsec:privacy zsum-avg}.

\subsection{Assumptions and lemmas}%
\label{subsec:privacy proof assumptions}

We start by proving the equivalent system matrix formulation.
\lemTemporalMatrixNotation*
\begin{proof} (\cref{lem:matrix notation consistency})
    We proceed by induction over $t\in\mathbb{N}$, using~\eqref{eq:matrix update rule} and unrolling the matrix multiplication.
    The initialization is done by definition.
    Now, assume that $\modelsvirtual{ni+k}{t} = \modelmatrix{i}{t}$ for all $k$. 

    First, we observe that we have $\modelsvirtual{ni+k}{t+\nicefrac{1}{2}} = \modelmatrix{i}{t+\nicefrac{1}{2}}$ by definition.    
    Second, we have, using~\eqref{eq:matrix update rule}:
    
    \begin{align*}
        \modelsvirtual{ni+k}{t+1}
        =& \left(\mixingmatrix\gossmatrixvirtual{}{t}(\modelsvirtual{}{t+\nicefrac{1}{2}} + \zsumnoisevirtual{}{t})\right)_{ni+k}
        \\=& \sum_{\hat{j}=0}^{n^2}\mixingmatrix_{ni+k,\hat{j}}\left(\gossmatrixvirtual{}{t}(\modelsvirtual{}{t+\nicefrac{1}{2}} + \zsumnoisevirtual{}{t})\right)_{\hat{j}}
    \end{align*}

    We can remove indexes in the virtual domain by exploiting the following properties of the matrices:
    \begin{itemize}
        \sloppy
        \item $\mixingmatrix_{n(i-1)+k,\hat{j}} \neq 0 \iff \hat{j}\in[\![n(i-1)+1, ni]\!]$ and $\mixingmatrix_{n(i-1)+k,\hat{j}} = 1$ in this case in such case, which simplifies the sum by removing $\mixingmatrix$.
        \item $\gossmatrixvirtual{j+n(i-1),\hat{u}}{t} \neq 0 \iff \hat{u} = n(j-1)+i$, in which case $\gossmatrixvirtual{j+n(i-1),\hat{u}}{t} = \gossmatrix{i,j}{t}$, which simplifies the sum further, by removing indexes and rewriting in terms of $\gossmatrix{}{t}$.
    \end{itemize}
    Thus, we get:
    \begin{align*}
        \modelsvirtual{ni+k}{t+1}
        =& \sum_{\hat{j}=n(i-1)+1}^{ni}\left(\gossmatrixvirtual{}{t}(\modelsvirtual{}{t+\nicefrac{1}{2}} + \zsumnoisevirtual{}{t})\right)_{\hat{j}}
        \\=& \sum_{j=1}^{n}\left(\gossmatrixvirtual{}{t}(\modelsvirtual{}{t+\nicefrac{1}{2}} + \zsumnoisevirtual{}{t})\right)_{n(i-1)+j}
        \\=& \sum_{j=1}^{n}\sum_{\hat{u}=0}^{n^2}\gossmatrixvirtual{n(i-1)+j,u}{t}\left(\modelsvirtual{}{t+\nicefrac{1}{2}} + \zsumnoisevirtual{}{t}\right)_{\hat{u}}
        \\=& \sum_{j=1}^{n}\gossmatrix{i,j}{t}\left(\modelsvirtual{}{t+\nicefrac{1}{2}} + \zsumnoisevirtual{}{t}\right)_{n(j-1) + i}
        \\=& \sum_{j=1}^{n}\gossmatrix{i,j}{t}\left(\modelmatrix{j}{t+\nicefrac{1}{2}} + \zsumnoise{j}{i}{t}\right)
    \end{align*}
    Where we used the induction hypothesis. 
    Now, we use the observation that $\modelsvirtual{ni+k}{t+\nicefrac{1}{2}} = \modelmatrix{i}{t+\nicefrac{1}{2}}$ to conclude our induction:
    \begin{align*}
        \modelsvirtual{ni+k}{t+1}
        =& \sum_{j=1}^{n}\gossmatrix{i,j}{t}\left(\modelmatrix{j}{t+\nicefrac{1}{2}} + \zsumnoise{j}{i}{t}\right)
        \\=& \modelmatrix{i}{t+1}.\qedhere
    \end{align*}
\end{proof}

\LemUnrolledDitribution*
Here, $\Sigma_Z^{(t)}$ represents correlated noises that will cancel out, 
for $t$ big enough we have ${(\mixingmatrix\gossmatrixvirtual{}{})}^t \Sigma_Z \transpose{(\mixingmatrix\gossmatrixvirtual{}{})}^t = 0$.
Thus, once $T$ is big enough, the variance $\Sigma_Z^{(t)}$ will become constant.
\begin{proof} (\cref{lem:unrolled distribution})
    We proceed by induction on $T$ for the expected value, and note $\Sigma_{T} = \sum_{t=1}^{T} {(1-\lr L)}^t {(\mixingmatrix\gossmatrixvirtual{}{})}^t \Sigma_Z^{(t)} \transpose{(\mixingmatrix\gossmatrixvirtual{}{})}^t$. 
    We have the following two update rules: 
    \begin{align*}
        \modelsvirtual{}{T+1} &= \mixingmatrix \gossmatrixvirtual{}{T}\left(\modelsvirtual{}{T} - \lr \nabla \loss{}(\modelsvirtual{}{T},\datapoint{}{T}) + \zsumnoisevirtual{}{t}\right)
        \\
        \unnoisedmodelvirtual{}{T+1} &= \mixingmatrix \gossmatrixvirtual{}{T}\left(\unnoisedmodelvirtual{}{T} - \lr \nabla \loss{}(\unnoisedmodelvirtual{}{T},\datapoint{}{T})\right).
    \end{align*}

    First, we can show by another induction that this is a linear combination of Gaussian random variables. 

    Then, let us look at the expected value for $\modelsvirtual{}{T+1}$: if we assume that the expected value of $\modelsvirtual{}{T}$ is $ \unnoisedmodelvirtual{}{T}$, 
    \cref{ass:gradient gaussian perturbation} guarantees that the expected value of $\modelsvirtual{}{T+1}$ is $ \unnoisedmodelvirtual{}{T+1}$.

    Finally, using~\cref{ass:gradient gaussian perturbation}, we have:

    \singlecolumnproof{
    \begin{align*}    
        \modelsvirtual{}{T+1} 
        \sim
        \mathcal{N}\left(
            \unnoisedmodelvirtual{}{T+1},
            L\Sigma_{T+1}
        \right)
    \end{align*}

    With the following update rule: 
    }{%
    \begin{align*}    
        \modelsvirtual{}{T+1} 
        \sim
        \mathcal{N}\left(
            \unnoisedmodelvirtual{}{T+1},
            L(1 - \lr L)\mixingmatrix\gossmatrixvirtual{}{T} \Sigma_{T} \transpose{(\mixingmatrix\gossmatrixvirtual{}{T}) + L(\mixingmatrix\gossmatrixvirtual{}{T})  \Sigma_Z\transpose{(\mixingmatrix\gossmatrixvirtual{}{T})}}
            \right)
    \end{align*}
    
    Thus, we obtain the update rule: 
    }%
    \begin{align*}
        L\Sigma_{T+1} &= (1 - \lr L)\mixingmatrix\gossmatrixvirtual{}{T} (L\Sigma_{T}) \transpose{(\mixingmatrix\gossmatrixvirtual{}{T})} 
        \\ &+ L(\mixingmatrix\gossmatrixvirtual{}{T})  \Sigma_Z\transpose{(\mixingmatrix\gossmatrixvirtual{}{T})}.
    \end{align*}
    
    This yields the following:
    \begin{align*}
        L\Sigma_{T} 
        =& L\sum_{t=1}^{T} {(1-\lr L)}^t {(\mixingmatrix\gossmatrixvirtual{}{})}^t \Sigma_Z^{(t)} \transpose{(\mixingmatrix\gossmatrixvirtual{}{})}^t
        \\=& L\gossmatrixtemporal{}{T}\Sigma_Z^{(t)}\transpose{\gossmatrixtemporal{}{T}}. \qedhere
    \end{align*}
    
\end{proof}

To prove \ac{PNDP}, we will need a bound between two adjacent inputs is derived using the following lemma: 
\begin{restatable}{lemma}{unnoisedexecnormbound}\label{lem:unnoised execution norm bound}
    Consider two unnoised executions. Then,
    \begin{align*}
        \norminf{\unnoisedmodelvirtual{}{t}- \unnoisedmodelpvirtual{}{t}}^2
        \leq& 
        \frac{4\lr^2\adjacentdatasetgradientbound^2}{1 + 4\lr^2L} ({(2 + 4\lr^2L)}^t -1).
    \end{align*}
\end{restatable}
This lemma bounds the maximal difference between local models of two adjacent unnoised executions. One limitation of this lemma is that it bounds over a maximum. This is because a gradient term must be isolated from the recursive term in the proof. To show this, \cref{subsec:privacy zsum-avg} focuses on the case where only averaging is performed, and no gradient descent. In this scenario, the equivalent of the above lemma is tighter, and we derive a generalization of previous results to our case of correlated noises.

\simplifiedproof{
\begin{proofsketch} (\cref{lem:unnoised execution norm bound})
    We use that $\norminf{\mixingmatrix\gossmatrixvirtual{}{}} = 1$, 
    and express $\norminf{\unnoisedmodelvirtual{}{t}- \unnoisedmodelpvirtual{}{t}}^2$ recursively using the triangle inequality and the sub-multiplicativity of the norm.
    We first get: 
    \begin{align*}
        &\norminf{\unnoisedmodelvirtual{}{t}- \unnoisedmodelpvirtual{}{t}}^2
        \leq
            2  \norminf{\unnoisedmodelvirtual{}{t-1} - \unnoisedmodelpvirtual{}{t-1}}^2
        \\  &+ 2\lr^2\underbrace{\norminf{\nabla\loss{}(\unnoisedmodelvirtual{}{t-1},\datapoint{}{t-1}) - \nabla\loss{}(\unnoisedmodelpvirtual{}{t-1},\datapointp{}{t-1})}^2}_{:=C^{(T)}_{1}}.
    \end{align*}
    Then we consider an intermediate term $\nabla\loss{}(\unnoisedmodelvirtual{}{t-1},\datapointp{}{t-1})$, and use the triangle inequality again. 
    The first term is bounded by input sensitivity~(\cref{ass:input_sensitivity}), and the second by smoothness of the loss function~(\cref{ass:smoothness}).
    This yields: 
    \begin{align*}
        C^{(T)}_{1} 
        \leq
        2 \adjacentdatasetgradientbound^2
        + 2 L\norminf{\unnoisedmodelvirtual{}{t-1} - \unnoisedmodelpvirtual{}{t-1}}^2
    \end{align*}
    Plugging this back into the equation and solving the recursion, we obtain the desired result.
\end{proofsketch}
}{%
\begin{proof} (\cref{lem:unnoised execution norm bound})
We know that $\norminf{\mixingmatrix\gossmatrixvirtual{}{}} = 1$.
    \singlecolumnproof{
        \begin{align*}
            \norminf{\unnoisedmodelvirtual{}{t}- \unnoisedmodelpvirtual{}{t}}^2
            \leq&
            \norminf{
                \mixingmatrix\gossmatrixvirtual{}{t}\left(
                \unnoisedmodelvirtual{}{t-\nicefrac{1}{2}}- \unnoisedmodelpvirtual{}{t-\nicefrac{1}{2}}
                \right)
            }^2
            \\\leq&
                2  \norminf{\unnoisedmodelvirtual{}{t-1} - \unnoisedmodelpvirtual{}{t-1}}^2
                + 2\lr^2C^{(t)}_{1},
        \end{align*}
        With $C^{(t)}_{1}:=\norminf{\nabla\loss{}(\unnoisedmodelvirtual{}{t-1},\datapoint{}{t-1}) - \nabla\loss{}(\unnoisedmodelpvirtual{}{t-1},\datapointp{}{t-1})}^2$.

    }{
    \begin{align*}
        \norminf{\unnoisedmodelvirtual{}{t}- \unnoisedmodelpvirtual{}{t}}^2
        \leq&
        \norminf{
            \mixingmatrix\gossmatrixvirtual{}{t}\left(
            \unnoisedmodelvirtual{}{t-\nicefrac{1}{2}}- \unnoisedmodelpvirtual{}{t-\nicefrac{1}{2}}
            \right)
        }^2
        \\\leq&
        \norminf{
            \unnoisedmodelvirtual{}{t-1} - \unnoisedmodelpvirtual{}{t-1}
            - \lr (\nabla\loss{}(\unnoisedmodelvirtual{}{t-1},\datapoint{}{t-1}) - \nabla\loss{}(\unnoisedmodelpvirtual{}{t-1},\datapointp{}{t-1} )  )
        }^2
        \\\leq&
            2  \norminf{\unnoisedmodelvirtual{}{t-1} - \unnoisedmodelpvirtual{}{t-1}}^2
            + 2\lr^2\underbrace{\norminf{\nabla\loss{}(\unnoisedmodelvirtual{}{t-1},\datapoint{}{t-1}) - \nabla\loss{}(\unnoisedmodelpvirtual{}{t-1},\datapointp{}{t-1})}^2}_{:=C^{(t)}_{1}}.
    \end{align*}
    }
    We focus on the left term, and notice that:
    \singlecolumnproof{
        \begin{align*}
            C^{(t)}_{1} 
            &\leq
            2\norminf{\nabla\loss{}(\unnoisedmodelvirtual{}{t-1},\datapoint{}{t-1}) - \nabla\loss{}(\unnoisedmodelvirtual{}{t-1},\datapointp{}{t-1})}^2 
            \\&+ 2 \norminf{\nabla\loss{}(\unnoisedmodelvirtual{}{t-1},\datapointp{}{t-1}) - \nabla\loss{}(\unnoisedmodelpvirtual{}{t-1},\datapointp{}{t-1})}^2
            \\&\stackrel{\eqref{eq:input_sensitivity_gradient},\eqref{eq:smoothness}}{\leq}
            2 \adjacentdatasetgradientbound^2
            + 2 L\norminf{\unnoisedmodelvirtual{}{t-1} - \unnoisedmodelpvirtual{}{t-1}}^2
        \end{align*}
    }{
    \begin{align*}
        C^{(t)}_{1} 
        &= 
        \norminf{\nabla\loss{}(\unnoisedmodelvirtual{}{t-1},\datapoint{}{t-1}) - \nabla\loss{}(\unnoisedmodelvirtual{}{t-1},\datapointp{}{t-1}) + \nabla\loss{}(\unnoisedmodelvirtual{}{t-1},\datapointp{}{t-1}) - \nabla\loss{}(\unnoisedmodelpvirtual{}{t-1},\datapointp{}{t-1})}^2
        \\&\leq
        2\norminf{\nabla\loss{}(\unnoisedmodelvirtual{}{t-1},\datapoint{}{t-1}) - \nabla\loss{}(\unnoisedmodelvirtual{}{t-1},\datapointp{}{t-1})}^2 
        + 2 \norminf{\nabla\loss{}(\unnoisedmodelvirtual{}{t-1},\datapointp{}{t-1}) - \nabla\loss{}(\unnoisedmodelpvirtual{}{t-1},\datapointp{}{t-1})}^2
        \\&\stackrel{\eqref{eq:input_sensitivity_gradient},\eqref{eq:smoothness}}{\leq}
        2 \adjacentdatasetgradientbound^2
        + 2 L\norminf{\unnoisedmodelvirtual{}{t-1} - \unnoisedmodelpvirtual{}{t-1}}^2
    \end{align*}
    }

    Thus, we get:
    \begin{align*}
        \norminf{\unnoisedmodelvirtual{}{t}- \unnoisedmodelpvirtual{}{t}}^2
        \leq&
        (2 + 4\lr^2L)  \norminf{\unnoisedmodelvirtual{}{t-1} - \unnoisedmodelpvirtual{}{t-1}}^2
        + 4\lr^2\adjacentdatasetgradientbound^2  
    \end{align*}

    Unrolling the recursion, we obtain:

    \begin{align*}
        \norminf{\unnoisedmodelvirtual{}{t}- \unnoisedmodelpvirtual{}{t}}^2
        \leq&
        \frac{4\lr^2\adjacentdatasetgradientbound^2}{1 + 4\lr^2L} ({(2 + 4\lr^2L)}^t -1). \qedhere
    \end{align*}
\end{proof}
}

\subsection{Proof of the main theorem}\label{subsec:privacy proof main theorem}
In this section, we remind and provide a full proof of~\cref{thm:zsumSgdPrivacy}.
\zsumSgdPrivacy*

\simplifiedproof{\todo{THERE NEEDS TO BE A PROOF HERE IF THEOREM ABOVE IS INCLUDED, SIMPLIFIED PROOF IS IN THE MAIN PAPER.}
}{%
\begin{proof}

    We want to bound the privacy loss that emerges from the view of nodes $V$. 
    To this end, we will use the matrix notations defined in~\cref{sec:matrix notations}, with a virtual network.

    For simplicity of notation, we assume that the communication matrix is fixed through time.
    The proof generalizes to arbitrary communication matrix at time $t$ at the expense of product notations.
    We obtain the following update rule for a given averaging round $t$:
    \begin{align}
        \modelsvirtual{}{t+1} = \mixingmatrix \gossmatrixvirtual{}{}\left(\modelsvirtual{}{t} - \lr \nabla \loss{}(\modelsvirtual{}{t},\datapoint{}{t}) + \zsumnoisevirtual{}{t}\right)
    \end{align}

    We now want to focus on two distinct executions on datasets $\datapoint{}{t}\sim_u \datapointp{}{t}$.
    The dot notation will correspond to the execution of the algorithm on an adjacent dataset.

    If we now consider some set of nodes $V\subseteq\nodeset$, we denote $\hat{V}\subseteq\nodesetvirtual$ the set of corresponding virtual nodes.
    \singlecolumnproof{We name  $P_V^{T}$ the privacy loss:
    \[P_V^{(T)} := \renyi{\pview{\hat{V}}{D}{T}}{\pview{\hat{V}}{\dot{D}}{T}}.\]
    We want to bound:
    \begin{align}
        \nonumber
        P_V^{(T)}
        &= \renyi{\pview{\hat{V}}{D}{T}}{\pview{\hat{V}}{\dot{D}}{T}}
        \\&\leq\sum_{t=0}^{T-1}\sum_{\hat{v}\in\hat{V}} \sum_{\hat{w}\in\neighborsvirtual{\hat{v}}{t}}
        \renyi{\modelsvirtual{\hat{w}}{t}}{\modelspvirtual{\hat{w}}{t}}
        \label{eq:globalprivacysplit}
    \end{align}
    }{
        We want to bound:
        \begin{align}
            \nonumber
            \renyi{\pview{V}{D}{T}}{\pview{V}{\dot{D}}{T}} 
            &= \renyi{\pview{\hat{V}}{D}{T}}{\pview{\hat{V}}{\dot{D}}{T}}
            \\&\leq\sum_{t=0}^{T-1}\sum_{\hat{v}\in\hat{V}} \sum_{\hat{w}\in\neighborsvirtual{\hat{v}}{t}}
            \renyi{\modelsvirtual{\hat{w}}{t}}{\modelspvirtual{\hat{w}}{t}}
            \label{eq:globalprivacysplit}
        \end{align}
    }%

    Our main focus is thus to bound $\renyi{\modelsvirtual{\hat{w}}{t}}{\modelspvirtual{\hat{w}}{t}}$.
    To this end, we want to apply~\cref{lem:unrolled distribution} to both $\modelsvirtual{\hat{w}}{t}$ and $\modelspvirtual{\hat{w}}{t}$.
    One key remark is that both their distributions are centered on slightly altered trajectories, corresponding to the two adjacent datasets.
    Thus, we apply~\cref{lem:unrolled distribution}, and obtain:
    \begin{align*}
        \modelsvirtual{\hat{w}}{t}\sim \mathcal{N}(\unnoisedmodelvirtual{\hat{w}}{t},L(\Sigma_{T})_{\hat{w},\hat{w}})
        ,&&
        \modelspvirtual{\hat{w}}{t}\sim \mathcal{N}(\unnoisedmodelpvirtual{\hat{w}}{t},L(\Sigma_{T})_{\hat{w},\hat{w}}),
    \end{align*}
    with $\Sigma_{T} = \sum_{t=1}^{T} (1-\lr L)^t (\mixingmatrix\gossmatrixvirtual{}{})^t \Sigma_Z \transpose{(\mixingmatrix\gossmatrixvirtual{}{})}^t$.

    One last thing we may want to do is factorize the noise expression:
    We now consider the matrix of all the noises $\tilde{Z}^{(T)}\in\mathbb{R}^{Tn^2}$, where $\tilde{Z}[tn^2 + \hat{w}]:=\zsumnoisevirtual{\hat{w}}{t}$ for $0\leq \hat{w} <n^2$.
    We can express the term by considering the temporal matrix notations of~\cref{sec:matrix notations}. 
    This leads to:
    \begin{align}
        \label{eq:unrolled models variance}
        \Sigma_T = \gossmatrixtemporal{}{T}\tilde{M} \Sigma_{\tilde{Y}} \transpose{(\gossmatrixtemporal{}{T}\tilde{M})}
    \end{align}

    Considering~\eqref{eq:parrallel executions distribution},\eqref{eq:unrolled models variance} along with \cref{lem:gaussian rényi bound},
    we obtain:
    \begin{align*}
        \renyi{\modelsvirtual{\hat{w}}{t}}{\modelspvirtual{\hat{w}}{t}} 
        \leq \frac{\alpha}{2L} 
            \frac{
                \norm{\unnoisedmodelvirtual{\hat{w}}{t}- \unnoisedmodelpvirtual{\hat{w}}{t}}^2
            }{
                \left(\gossmatrixtemporal{}{t}\tilde{M} \Sigma_{\tilde{Y}} \transpose{(\gossmatrixtemporal{}{t}\tilde{M})}\right)_{\tilde{w},\tilde{w}}
            }
    \end{align*}

    Finally, we need to bound the difference between the two unnoised executions $\norm{\unnoisedmodelvirtual{\hat{w}}{t}- \unnoisedmodelpvirtual{\hat{w}}{t}}^2$ 
    using~\cref{lem:unnoised execution norm bound}.

    Putting it all together in~\eqref{eq:globalprivacysplit}, we can bound:
    \singlecolumnproof{
        \begin{align}
            P_V^{(T)}
            \leq& 
            \frac{2\alpha\lr^2\adjacentdatasetgradientbound^2}{L + 4\lr^2L^2}\sum_{t=0}^{T-1}\sum_{\hat{v}\in\hat{V}} \sum_{\hat{w}\in\neighborsvirtual{\hat{v}}{t}}
            \frac{{(2 + 4\lr^2L)}^t -1
            }{
            {\left(
                \gossmatrixtemporal{}{t}\tilde{M}\Sigma_{\tilde{Y}}\transpose{(\gossmatrixtemporal{}{t}\tilde{M})} 
            \right)}_{\tilde{w},\tilde{w}}
            }.
        \end{align}
    }{
        \begin{align}
            \renyi{\pview{V}{D}{T}}{\pview{V}{\dot{D}}{T}} 
            \leq& 
            \frac{2\alpha\lr^2\adjacentdatasetgradientbound^2}{L + 4\lr^2L^2}\sum_{t=0}^{T-1}\sum_{\hat{v}\in\hat{V}} \sum_{\hat{w}\in\neighborsvirtual{\hat{v}}{t}}
            \frac{{(2 + 4\lr^2L)}^t -1
            }{
            {\left(
                \gossmatrixtemporal{}{t}\tilde{M}\Sigma_{\tilde{Y}}\transpose{(\gossmatrixtemporal{}{t}\tilde{M})} 
            \right)}_{\tilde{w},\tilde{w}}
            }.
        \end{align}
    }
\end{proof}
}

\subsection{Proof of the averaging algorithm}%
\label{subsec:zsumsgd-pndp proof}
We prove~\cref{subsec:privacy zsum-avg}.

\ThmPndpZsumAvg*
\begin{proofsketch} (\cref{thm:zsum-avg pndp})
We can follow the same proof concept for the averaging algorithm presented in~\cref{alg:ZeroSum}. 
In this case, the notion of adjacent dataset is slightly different, as it concerns the original data itself $\modelmatrix{}{0}$.
We will obtain a simpler update rule:
\begin{align*}
    \modelsvirtual{}{T+1} = \mixingmatrix\gossmatrixvirtual{}{T}\left(\modelsvirtual{}{T} + \zsumnoisevirtual{}{T}\right).
\end{align*}

Unrolling the model updates, and following a similar reasoning, we obtain that:
\begin{align*}
    \modelsvirtual{}{T+1} &\sim \mathcal{N}({(\mixingmatrix\gossmatrixvirtual{}{})}^T\modelsvirtual{}{0}, \tilde{W}^{T}\tilde{M}\Sigma_Y\transpose{(\tilde{W}^{T}\tilde{M})})
    \\\text{where } \gossmatrixtemporal{}{T}&:= 
    \begin{pNiceArray}{ccc}
        (\mixingmatrix\gossmatrixvirtual{}{})^T,
        & \dots, 
        & \mixingmatrix\gossmatrixvirtual{}{}
    \end{pNiceArray}
    \in \mathbb{R}^{n^2\times Tn^2}
\end{align*}
Then, using the same decomposition and~\cref{lem:gaussian rényi bound}, we observe the sensitivity is:
\begin{align*}
    \norm{\left((\mixingmatrix\gossmatrixvirtual{}{})^T\left(\modelsvirtual{}{0} - \modelspvirtual{}{0}\right)\right)_{\hat{w}}}^2
    \leq \left((\mixingmatrix\gossmatrixvirtual{}{})^T\right)_{\hat{w},\hat{u}}\Delta^2,
\end{align*} 
with $\Delta$ the bound on two adjacent datasets, since $\modelsvirtual{}{0}$ and $\modelspvirtual{}{0}$ are only different in component $u$.
We can derive the desired result from this.
\end{proofsketch} 

\section{Proofs of \sys main properties}
\label{sec:zsum small properties proofs}
This section contains proofs to~\cref{sec:zsum algorithm and properties}.
\simplifiedproof{%
}{%
\noiseCancellation*
\begin{proof}
    Using the notation in \cref{alg:ZeroSum}, and since the matrix is symmetric, we have for a fixed node $a$: 
    \begin{align*}
        \sum_{v\in \neighbors{a}{}} \gossmatrix{a,v}{} \zsumnoise{a}{v}{}
        &= \sum_{v\in \neighbors{a}{}} \gossmatrix{a,v}{} [\ynoise{a}{v}{} -  \frac{1}{\nodedegree{a}\gossmatrix{a,v}{}}\sum_{j\in\neighbors{a}{}}\gossmatrix{a,j}{} \ynoise{a}{j}{}]
        \\&= \sum_{v\in \neighbors{a}{}} \gossmatrix{a,v}{}\ynoise{a}{v}{} - \sum_{v\in \neighbors{a}{}} \frac{1}{\nodedegree{a}}(\sum_{j\in\neighbors{a}{}} \gossmatrix{a,j}{}\ynoise{a}{j}{})
        \\&= \sum_{v\in \neighbors{a}{}} \gossmatrix{a,v}{}\ynoise{a}{v}{} - \sum_{j\in\neighbors{a}{}} \gossmatrix{a,j}{}\ynoise{a}{j}{}
        \\&= 0.
    \end{align*}
\end{proof}
}

\simplifiedproof{%
\begin{proofsketch}(\cref{lem:conservation average model}):
    We use the definition of average model, and expand the model update:
    \begin{align}
        \nonumber
        \avgmodel{t+1} 
        =& 
        \frac{1}{n}\sum_{a=1}^n \sum_{v\in \neighbors{a}{}}\gossmatrix{a,v}{t} \model{v}{t+\nicefrac{1}{2}}  
        \\&+ \frac{1}{n}\sum_{a=1}^n \sum_{v\in \neighbors{a}{}}\gossmatrix{a,v}{t} \zsumnoise{v}{a}{t}
        \label{sketcheq:noise_propagation}
    \end{align}

    For the first term:
    \begin{align*}
        \frac{1}{n}\sum_{a=1}^n \sum_{v\in \neighbors{a}{}}\gossmatrix{a,v}{t} \model{v}{t+\nicefrac{1}{2}} 
        &= \frac{1}{n}\sum_{a=1}^n \gossmatrix{a}{t}\model{}{t+\nicefrac{1}{2}}
        \\ &= \frac{1}{n} \transpose{\1}\model{}{t+\nicefrac{1}{2}}
        \\ &= \avgmodel{t+\nicefrac{1}{2}}
    \end{align*}
    Where we used the properties of the mixing matrix. 
    Focusing on the second term, we obtain:
    \begin{align*}
        \frac{1}{n}\sum_{a=1}^n \sum_{v\in \neighbors{a}{}}\gossmatrix{a,v}{t} \zsumnoise{v}{a}{t}
        &= \frac{1}{n}\sum_{a=1}^n \sum_{v = 1}^n \gossmatrix{a,v}{t} \zsumnoise{v}{a}{t}
        \\&= \frac{1}{n}\sum_{v=1}^n \sum_{a = 1}^n \gossmatrix{a,v}{t} \zsumnoise{v}{a}{t}
        \\&=0.
    \end{align*}
    Plugging this into~\eqref{sketcheq:noise_propagation} yields the desired result:
    \begin{align*}
        \avgmodel{t+1} = \frac{1}{n}\sum_{a=1}^n \model{a}{t+1} = \avgmodel{t+\nicefrac{1}{2}}
    \end{align*}
\end{proofsketch}
}{%
\conservationAvgModel*

\begin{proof}
    
    \begin{align}
        \nonumber
        \avgmodel{t+1} &= 
        \frac{1}{n}\sum_{a=1}^n \model{a}{t+1}
        = \frac{1}{n}\sum_{a=1}^n \sum_{v\in \neighbors{a}{}}\gossmatrix{a,v}{t} (\model{v}{t+\nicefrac{1}{2}}  + \zsumnoise{v}{a}{t})
        \\&=\frac{1}{n}\sum_{a=1}^n \sum_{v\in \neighbors{a}{}}\gossmatrix{a,v}{t} \model{v}{t+\nicefrac{1}{2}}  + \frac{1}{n}\sum_{a=1}^n \sum_{v\in \neighbors{a}{}}\gossmatrix{a,v}{t} \zsumnoise{v}{a}{t}
        \label{eq:noise_propagation}
    \end{align}

    For the first term:
    \begin{align*}
        \frac{1}{n}\sum_{a=1}^n \sum_{v\in \neighbors{a}{}}\gossmatrix{a,v}{t} \model{v}{t+\nicefrac{1}{2}} 
        &= \frac{1}{n}\sum_{a=1}^n \gossmatrix{a}{t}\model{}{t+\nicefrac{1}{2}}
        \\ &= \frac{1}{n} \transpose{\1}\model{}{t+\nicefrac{1}{2}}
        \\ &= \avgmodel{t+\nicefrac{1}{2}}
    \end{align*}
    Where we used the properties of the mixing matrix. 

    Focusing on the second term in~\eqref{eq:noise_propagation}, we obtain:
    \begin{align*}
        \frac{1}{n}\sum_{a=1}^n \sum_{v\in \neighbors{a}{}}\gossmatrix{a,v}{t} \zsumnoise{v}{a}{t}
        &= \frac{1}{n}\sum_{a=1}^n \sum_{v = 1}^n \gossmatrix{a,v}{t} \zsumnoise{v}{a}{t}
        \\&= \frac{1}{n}\sum_{v=1}^n \sum_{a = 1}^n \gossmatrix{a,v}{t} \zsumnoise{v}{a}{t}
        \\&=0.
    \end{align*}
    Plugging this into~\eqref{eq:noise_propagation} yields the desired result:
    \begin{align*}
        \avgmodel{t+1} = \frac{1}{n}\sum_{a=1}^n \model{a}{t+1} = \avgmodel{t+\nicefrac{1}{2}}
    \end{align*}
\end{proof}
}

\simplifiedproof{
\begin{proofsketch} (\cref{lem:corrected zsum variance})
    First, looking at the definition of $\zsumnoise{a}{v}{}$, we obtain that:
    \begin{align}
        \nonumber\zsumnoise{a}{v}{} 
        &= \ynoise{a}{v}{} -  \frac{1}{\nodedegree{a} \gossmatrix{a,v}{}}\sum_{j\in\neighbors{a}{}}\gossmatrix{a,j}{} \ynoise{a}{j}{}
        \\&= \frac{\nodedegree{a} - 1 }{\nodedegree{a}}\ynoise{a}{v}{} -  \frac{1}{\nodedegree{a}\gossmatrix{a,v}{}}\sum_{\substack{j\in\neighbors{a}{}\\j\neq v}}\gossmatrix{a,j}{} \ynoise{a}{j}{}
        \label{sketcheq:independent noise separation}
    \end{align}
    Thus, $\zsumnoise{a}{v}{}$ is a linear combination of independent Gaussian noises. This means that $\zsumnoise{a}{v}{}$ also follows a Gaussian distribution.
    Since the mean of all $\ynoise{a}{v}{}$ is $0$, so is the mean of $\zsumnoise{a}{v}{}$.

    To obtain the desired result, we look at the variance. Using~\eqref{sketcheq:independent noise separation} and variance properties on independant random variables, we obtain:   
    \begin{align*}
        \Var(\zsumnoise{a}{v}{}) 
        &= \Var(\frac{\nodedegree{a} - 1 }{\nodedegree{a}}\ynoise{a}{v}{} -  \frac{1}{\nodedegree{a}\gossmatrix{a,v}{}}\sum_{\substack{j\in\neighbors{a}{}\\j\neq v}}\gossmatrix{a,j}{} \ynoise{a}{j}{})
        \\&= \left(\frac{(\nodedegree{a} - 1) ^2}{\nodedegree{a}^2} +\frac{\sum_{\substack{j\in\neighbors{a}{}\\j\neq v}}(\gossmatrix{a,j}{})^2 }{(\nodedegree{a}\gossmatrix{a,v}{})^2}\right)\lr^2\ynoisevar{a}
    \end{align*}   
\end{proofsketch}
}{%
\correctedVarianceZsum*
\begin{proof}
    First, looking at the definition of $\zsumnoise{a}{v}{}$, we obtain that:
    \begin{align}
        \nonumber\zsumnoise{a}{v}{} 
        &= \ynoise{a}{v}{} -  \frac{1}{\nodedegree{a} \gossmatrix{a,v}{}}\sum_{j\in\neighbors{a}{}}\gossmatrix{a,j}{} \ynoise{a}{j}{}
        \\&= \frac{\nodedegree{a} - 1 }{\nodedegree{a}}\ynoise{a}{v}{} -  \frac{1}{\nodedegree{a}\gossmatrix{a,v}{}}\sum_{\substack{j\in\neighbors{a}{}\\j\neq v}}\gossmatrix{a,j}{} \ynoise{a}{j}{}
        \label{eq:independent noise separation}
    \end{align}
    Thus, $\zsumnoise{a}{v}{}$ is a linear combination of independent Gaussian noises. This means that $\zsumnoise{a}{v}{}$ also follows a Gaussian distribution.
    Since the mean of all $\ynoise{a}{v}{}$ is $0$, so is the mean of $\zsumnoise{a}{v}{}$.
    
    To obtain the desired result, we only need to look at the variance. Using~\eqref{eq:independent noise separation}, we obtain:   
    \begin{align*}
        \Var(\zsumnoise{a}{v}{}) 
        &= \Var(\frac{\nodedegree{a} - 1 }{\nodedegree{a}}\ynoise{a}{v}{} -  \frac{1}{\nodedegree{a}\gossmatrix{a,v}{}}\sum_{\substack{j\in\neighbors{a}{}\\j\neq v}}\gossmatrix{a,j}{} \ynoise{a}{j}{})
        \\&= (\frac{\nodedegree{a} - 1 }{\nodedegree{a}})^2\Var(\ynoise{a}{v}{})   + (\frac{1}{\nodedegree{a}\gossmatrix{a,v}{}})^2  \Var\left(\sum_{\substack{j\in\neighbors{a}{}\\j\neq v}}\gossmatrix{a,j}{} \ynoise{a}{j}{}\right)
        \\&= (\frac{\nodedegree{a} - 1 }{\nodedegree{a}})^2\lr^2\ynoisevar{a} + (\frac{1}{\nodedegree{a}\gossmatrix{a,v}{}})^2  \sum_{\substack{j\in\neighbors{a}{}\\j\neq v}}(\gossmatrix{a,j}{})^2\lr^2 \ynoisevar{a}
        \\&= \left(\frac{(\nodedegree{a} - 1) ^2}{\nodedegree{a}^2} +\frac{\sum_{\substack{j\in\neighbors{a}{}\\j\neq v}}(\gossmatrix{a,j}{})^2 }{(\nodedegree{a}\gossmatrix{a,v}{})^2}\right)\lr^2\ynoisevar{a}
    \end{align*}
\end{proof}
}

\section{Convergence rate of \sys}%
\label{sec:convergence proofs}

\subsection{Useful inequalities}
\begin{lemma}\label{lem:norm_of_sum_reduction}
  For any set of $n$ vectors $(a_i)_{i=1}^{n}, a_i\in \modelspace$:
  \begin{align*}
    \norm{\sum_{i=1}^{n}a_i}^2
    \leq n \sum_{i=1}^{n}\norm{a_i}^2
  \end{align*}
\end{lemma}

\begin{lemma}\label{lem:norm_doublesum_beta_majoration}
  For any vectors $\mathbf{a},\mathbf{b}\in\modelspace$, for any $\beta>0$, we have:
  \begin{align*}
    \norm{\mathbf{a}+\mathbf{b}}^2\leq (1+\beta)\norm{\mathbf{a}}^2 + (1 + \beta^{-1})\norm{\mathbf{b}}
  \end{align*}

\end{lemma}

\subsection{Convergence rate results}\label{subsec:main convergence proof}

\simplifiedproof{}{\zsumConvergenceRate*}
\begin{proof}(\cref{thm:convergence_rate})
  We used a similar situation to~\cite{koloskovaUnifiedTheoryDecentralized2020} with $\tau = 1$ and a fixed communication matrix sampling distribution.
  The proof follows the same structure as in their paper.
  Our algorithm only induces some changes in some of the intermediary lemmas that need to be adapted to obtain the main result.

  To this end, we restate~\cref{prop:mini_batch variance,lem:descent lemma,lem:rec consensus distance} in our setting.
  We can then solve the main equation in the following manner:

  \begin{itemize}
    \item We bound the distance of the averaged model to the optimum~\cref{lem:descent lemma}.
          It is the case $r_t = \expect{\norm{\avgmodel{t} - \optimum}^2}$, $e_t = \sampledloss(\avgmodel{t}) - \sampledloss(\optimum), a =\frac{\mu}{2}, b= 1, c = \frac{\gradientnoiseavg}{n}$ and $B = 3L$
    \item We also bound the consensus distance with a recursive bound using~\cref{lem:rec consensus distance}. The next step is to determine the precise constants to continue the proof.
  \end{itemize}

  The equation of the consensus distance (\cref{lem:rec consensus distance}) is of the following form:

  \simplifiedproof{
    \begin{align*}
      \consdistance{t}
      \leq & (1+\beta)(1-\frac{7p}{16})\consdistance{t-1}
      +  (1+\beta)D\lr^2 e_{t-1}
      \\& + \lr^2(1+\beta)A
      \\&+ \lr^2(1 + \beta^{-1})\frac{\modelsize}{n}\sum_{i=1}^{n} \nodedegree{i} \sum_{v= 1}^{n} \frac{(\nodedegree{v} - 1) ^2}{\nodedegree{v}}\ynoisevar{v}
    \end{align*}
  }{%
    \begin{align*}
      \consdistance{t}
      \leq & (1+\beta)(1-\frac{7p}{16})\consdistance{t-1}
      +  (1+\beta)D\lr^2 e_{t-1}
      \\ &+ \left( (1+\beta)A
      + (1 + \beta^{-1})\frac{\modelsize}{n}\sum_{i=1}^{n} \nodedegree{i} \sum_{v= 1}^{n} \frac{{(\nodedegree{v} - 1)}^2}{\nodedegree{v}}\ynoisevar{v}\right)\lr^2
    \end{align*}
  }
  with $e_{t} = \localloss{}(\avgmodel{t}) - \localloss{}(\optimum)$, $D = \frac{36L}{p}$ and  $A = \gradientnoiseavg+ \frac{18}{p}\gradientnormavg$

  Because of the $1+\beta$ factor, we cannot directly apply the recursion-solving Lemma to our scenario (Lemma 12 in~\cite{koloskovaUnifiedTheoryDecentralized2020}).
  We can however modify our current equation to match the beginning of their proof of this Lemma.
  This is mostly possible because we are in the case $\tau = 1$, meaning that we require a slightly stronger property on the matrices' distribution.

  We can now rewrite the previous equation by setting $\beta = \frac{3p}{16-7p}$ (rq: we only require $\beta>0$, which is satisfied since $0\leq p \leq 1$),
  \begin{align*}
    (1+\beta)
    = & \frac{16-7p+3p}{16-7p} = \frac{16-4p}{16-7p}
  \end{align*}
  and
  \begin{align*}
    (1+\beta)(1-\frac{7p}{16})
    = & \frac{16-4p}{16-7p} \frac{16-7p}{16}
    = \frac{16-4p}{16}
    = 1 - \frac{p}{4}
  \end{align*}
  Putting these inside the main equation, and setting
  \begin{align*}
    A' & =\left((1+\beta)A + (1 + \beta^{-1})\frac{\modelsize}{n}\sum_{i=1}^{n} \nodedegree{i} \sum_{v= 1}^{n} \frac{(\nodedegree{v} - 1) ^2}{\nodedegree{v}}\ynoisevar{v}\right)\lr
    \\D' &= \frac{1}{2}(1+\beta)D = \frac{16-4p}{2(16-7p)}\frac{36L}{p}
  \end{align*}
  we obtain:
  \begin{align*}
    \consdistance{t}
    \leq (1 - \frac{p}{4})\consdistance{t-1}
    + 2D'\lr^2 e_{t-1}
    +  2A' \lr^2
  \end{align*}
  This is exactly the term obtained in~\cite{koloskovaUnifiedTheoryDecentralized2020}'s Lemma 12 after unrolling the different terms, which is only needed when $\tau>1$.
  Thus, in our case, we can fall back to their proof using this approach. We just need to ensure Lemma 12's hypothesis are verified:
  \begin{itemize}
    \item $0<p\leq 1$
    \item $\tau = 1 \geq 1$
    \item $A',D' \geq 0$
    \item $\set{\lr^2}_{t\leq0}$ is a $\frac{8}{p}$-slow decreasing sequence since it is a constant.
    \item $\set{w_t:= {(1-a\lr)}^{-(t+1)}}$ is a $\frac{16}{p}$-slow increasing sequence of weights.
  \end{itemize}

  Thus, we can have the same reasoning as the proof of Lemma 12 in~\cite{koloskovaUnifiedTheoryDecentralized2020}, and obtain the lemma's result with the following equation:

  \begin{align}
    B\sum_{t=0}^{T}w_t \consdistance{t} \leq \frac{b}{2} \sum_{t=0}^{T}w_t e_t + 64A'B\lr^2\sum_{t=0}^{T}w_t
  \end{align}
  for some constant E and stepsize $\lr\leq \frac{1}{16}\sqrt{\frac{pb}{D'B}}$

  From this point on, we can follow the exact ending of the proof, the only difference are our new constants $A'$ and $D'$. 
  We thus obtain:
    \begin{align*}
      \frac{1}{2W_T}\sum_{t=0}^{T}b{w}_{t}{e}_{t}
      \leq & \frac{1}{W_T}\sum_{t=0}^{T}\left(\frac{(1-a\lr)w_t}{\lr}r_t - \frac{w_t}{\lr}r_{t+1}\right)
      \\&+ \frac{c}{W_T}\sum_{t=0}^{T}w_t\lr + \frac{64BA'}{W_T}\sum_{t=0}^{T}w_t\lr^2
    \end{align*}
  
  (with $W_T= \sum_{t=0}^{T}w_t$).

  Finally, we use Lemma 13 of~\cite{koloskovaUnifiedTheoryDecentralized2020} to obtain the final result, since we verify the following hypothesis:
  $a,b>0, c,A',B\geq 0$

  Thus, we obtain that for a well-chosen $\lr$:
  \[\frac{1}{2W_T}\sum_{t=0}^{T}be_tw_t + ar_{T+1} \leq \mathcal{O}\left(r_0d\text{exp}\left[-\frac{a(T+1)}{d}\right] + \frac{c}{aT} + \frac{BA'}{a^2T^2}\right).\]
  Plugging in the values yields the result for~\cref{thm:convergence_rate}.

\end{proof}

From the previous result, we also prove the convergence rate to an arbitrary $\rho$ accuracy:

\zsumConvergenceRateEpsilon*
\begin{proof}
For \cref{alg:ZeroSum-sgd} to reach the target accuracy $\rho$, we need to have:
\begin{align}
 & \frac{1}{2W_T}\sum_{t=0}^{T}w_t \left(\expect{\sampledloss(\avgmodel{t})} - \sampledloss^* \right) + \frac{\mu}{2} r_{T+1}\leq \rho\label{eq:accuracy_convergence_target}
\end{align}
However, from \cref{thm:convergence_rate}, we know that
  \begin{align}
     & \frac{1}{2W_T}\sum_{t=0}^{T}w_t \left(\expect{\sampledloss(\avgmodel{t})}- \sampledloss^* \right) + \frac{\mu}{2} r_{T+1}\nonumber                        \\
     & \leq\kappa\left(\frac{r_0L}{p}\exp\left[-\frac{\mu p(T+1)}{192\sqrt{3}L}\right] + \frac{\gradientnoiseavg}{n\mu T} + \frac{LA'}{\mu^2T^2}\right)\nonumber \\
     & \text{for some constant $\kappa>0$.}\nonumber
  \end{align}
Thus, in order to satisfy \eqref{eq:accuracy_convergence_target}, it suffices to simultaneously have:
\begin{align}
       & \kappa \frac{r_0L}{p}\exp\left[-\frac{\mu p(T+1)}{192\sqrt{3}L}\right] \leq \frac{\rho}{3}\nonumber                  \\
  \iff & \exp\left[\frac{\mu p(T+1)}{192\sqrt{3}L}\right]\geq \frac{3\kappa r_0L}{\rho p}\nonumber                            \\
  \iff & T\geq \frac{192\sqrt{3}L}{\mu p}\ln{\left[\frac{3\kappa r_0L}{\rho p}\right]}-1,\label{eq:accuracy_convergence_suff1}
\end{align}

\begin{align}
   & \kappa\frac{\gradientnoiseavg}{n\mu T}\leq \frac{\rho}{3}
  \iff T \geq \frac{3\kappa \gradientnoiseavg}{n\mu \rho },\label{eq:accuracy_convergence_suff2}
\end{align}
and 
\begin{align}
   & \kappa\frac{LA'}{\mu^2T^2}\leq \frac{\rho}{3}
  \iff T\geq \sqrt{\frac{3\kappa LA'}{\rho\mu^2}}.\label{eq:accuracy_convergence_suff3}
\end{align}

  Therefore, in order to simultaneously satisfy the inequalities in \eqref{eq:accuracy_convergence_suff1},\eqref{eq:accuracy_convergence_suff2}, and \eqref{eq:accuracy_convergence_suff3}, it suffices to have
  \begin{align*}
             & T\geq \frac{192\sqrt{3}L}{\mu p}\ln{\left[\frac{3\kappa r_0L}{\rho p}\right]}-1+ \frac{3\kappa \gradientnoiseavg}{n\mu \rho }+\sqrt{\frac{\kappa LA'}{3\mu^2}}\nonumber \\
    \implies & T> \frac{192\sqrt{3}L}{\mu p}\ln{\left[\frac{3\kappa r_0L}{\rho p}\right]}+ \frac{3\kappa \gradientnoiseavg}{n\mu \rho }+\sqrt{\frac{3\kappa LA'}{\rho\mu^2}}\qedhere
  \end{align*}
\end{proof}
\subsection{Intermediary lemmas proofs}
We state and prove the necessary lemmas for the convergence proof of~\cref{subsec:main convergence proof}.

\begin{restatable}{proposition}{propMiniBatchVariance}\label{prop:mini_batch variance}
    \emph{Mini-batch variance} (Proposition 5 in~\cite{koloskovaUnifiedTheoryDecentralized2020})
    Assume that $\loss{i}$ is $L$-smooth (\cref{ass:smoothness}) with bounded noise at the optimum (\cref{ass:bounded gradient noise}). 
    Then, for any $i\in[\![1,n]\!]$, we have:
    \begin{align*}
        \noindent\condexpect{\datapoint{1}{},\dots,\datapoint{n}{}}{\norm{\frac{1}{n} \sum_{i=1}^{n}(\nabla \sampledloss(\model{i}{}) - \nabla \loss{i}(\model{i}{}, \datapoint{i}{}))}^2}
        \\\leq \frac{3L^2}{n}\sum_{i=1}^{n}\norm{\model{i}{} - \avgmodel{}}^2 + 6L(\localloss{}{}(\avgmodel{}) - \sampledloss{}(\optimum)) + 3\gradientnoiseavg.
    \end{align*}
\end{restatable}

\begin{proof} (\cref{prop:mini_batch variance})
  Nothing changes in this proof compared to the original work, since only the gradient and the loss functions are needed, and averaging rounds are not considered.
\end{proof}

\begin{restatable}{lemma}{lemDescent}\label{lem:descent lemma}
    \emph{Descent lemma for convex cases}. (Lemma 8 of~\cite{koloskovaUnifiedTheoryDecentralized2020})
    Under~\cref{ass:smoothness,ass:mu-convexity,ass:bounded gradient noise,ass:expected consensus rate},
    with stepsize $\lr\leq\frac{1}{12L}$ we have:
    \singlecolumnproof{
    \begin{align*}
        \condexpect{\datapoint{1}{t},\dots,\datapoint{n}{t}}{\norm{\avgmodel{t+1} - \optimum}^2}
        \leq& (1 - \frac{\lr \mu}{2})\norm{\avgmodel{t} - \optimum}^2
        \\&+ \frac{\lr^2\gradientnoiseavg}{n} 
        \\&- \lr(\localloss{}(\avgmodel{t}) - \localloss{}(\optimum)) 
        \\&+ \gamma\frac{3L}{n}\sum_{i=1}^{n}\norm{\avgmodel{t} - \model{i}{t}}^2.
    \end{align*}
    }{%
    \begin{align*}
        \condexpect{\datapoint{1}{t},\dots,\datapoint{n}{t}}{\norm{\avgmodel{t+1} - \optimum}^2}
        \leq& (1 - \frac{\lr \mu}{2})\norm{\avgmodel{t} - \optimum}^2
        \\&+ \frac{\lr^2\gradientnoiseavg}{n} 
        - \lr(\localloss{}(\avgmodel{t}) - \localloss{}(\optimum)) 
        \\&+ \gamma\frac{3L}{n}\sum_{i=1}^{n}\norm{\avgmodel{t} - \model{i}{t}}^2.
    \end{align*}
    }
\end{restatable}%
\begin{proof} (\cref{lem:descent lemma})
  Because of \sys's properties (in particular~\cref{lem:conservation average model}), this property holds almost immediately from Lemma 8 of~\cite{koloskovaUnifiedTheoryDecentralized2020}.
  Using~\cref{lem:conservation average model}, we have:
  \begin{align*}
    \norm{\avgmodel{t+1} - \optimum}^2
    = & \norm{\avgmodel{t+\nicefrac{1}{2}} - \optimum}^2
    \\=& \norm{\avgmodel{t} - \frac{\lr}{n}\sum_{i=1}^{n}\nabla\loss{i}(\model{i}{t},\datapoint{i}{t}) - \optimum}^2
  \end{align*}
  This corresponds to the first line of Lemma 8, so following the proof will yield the same result. More generally, this property would not hold as it stands for a method that only cancels the noise in expectation: because we consider a norm here, this will lead to an additional term equal to the variance of the residual noise on the network, e.g. the variance of the sum of all the noises. If the noises are not correlated, this is an estimator of the original distribution, yielding an additional term. In our case, this term is exactly zero.

\end{proof}

\lemConsensusDistance*
This lemma has an additional last term compared to state-of-the-art \ac{DL}~\cite{koloskovaUnifiedTheoryDecentralized2020}.
It stems from the presence of noise, that shifts local models away from the true average.
\begin{proof} (\cref{lem:rec consensus distance})
  \begin{align*}
    n\consdistance{t}
    = & \sum_{i=1}^{n}\condexpect{t}{\norm{\model{i}{t} - \avgmodel{t}}^2}
    \\= & \sum_{i=1}^{n}\condexpect{t}{\norm{(\model{i}{t} - \avgmodel{t-1}) - (\avgmodel{t} - \avgmodel{t-1})}^2}
    \\\leq& \sum_{i=1}^{n}\condexpect{t}{\norm{(\model{i}{t} - \avgmodel{t-1})}^2}
  \end{align*}
  Where we used that $\sum_{i=1}^{n}\norm{a_i-\bar{a}}^2 \leq \sum_{i=1}^{n}\norm{a_i}^2$.
  Unrolling the model update:
  \begin{align*}
    \model{i}{t}
    = & \sum_{v\in \neighbors{i}{t-1}}\gossmatrix{i,v}{t-1} (\model{v}{t-\nicefrac{1}{2}}  + \zsumnoise{v}{i}{t-1})
    \\= & \sum_{v\in \neighbors{i}{t-1}}\gossmatrix{i,v}{t-1} ((\model{v}{t-1} - \lr \nabla \loss{v}(\model{v}{t-1},\datapoint{v}{t-1}))
    + \zsumnoise{v}{i}{t-1})
    \\= & \sum_{v\in \neighbors{i}{t-1}}(\gossmatrix{i,v}{t-1}(\model{v}{t-1}))
    + \sum_{v\in \neighbors{i}{t-1}}(\gossmatrix{i,v}{t-1} \zsumnoise{v}{i}{t-1})
    \\ & -\sum_{v\in \neighbors{i}{t-1}}(\gossmatrix{i,v}{t-1} \lr \nabla\loss{v}(\model{v}{t-1},\datapoint{v}{t-1}))
  \end{align*}

  This yields, after expanding the recursion and using \cref{lem:norm_doublesum_beta_majoration}, for any $\beta>0$:
  \singlecolumnproof{%
    \begin{align*}
      n\consdistance{t}
      \leq & (1+\beta)\underbrace{\sum_{i=1}^{n}\condexpect{t}{\norm{\sum_{v\in \neighbors{i}{t-1}}T_3}^2}}_{:= T_1}
      \\ &+(1+\beta^{-1})\underbrace{\sum_{i=1}^{n}\condexpect{t}{\norm{\sum_{v\in \neighbors{i}{t-1}}(\gossmatrix{i,v}{t-1} \zsumnoise{v}{i}{t-1})}^2}}_{:= T_2}
    \end{align*}
    where we have
    \begin{align*}
      T_3 := \gossmatrix{i,v}{t-1}\left(\model{v}{t-1} - \lr \nabla\loss{v}(\model{v}{t-1},\datapoint{v}{t-1})\right) - \avgmodel{t-1}
    \end{align*}
  }{%
    \begin{align*}
      n\consdistance{t}
      \leq & \sum_{i=1}^{n}\condexpect{t}{\norm{\sum_{v\in \neighbors{i}{t-1}}(\gossmatrix{i,v}{t-1}\model{v}{t-1})- \avgmodel{t-1} - \sum_{v\in \neighbors{i}{t-1}}(\gossmatrix{i,v}{t-1} \lr \nabla\loss{v}(\model{v}{t-1},\datapoint{v}{t-1}))  + \sum_{v\in \neighbors{i}{t-1}}(\gossmatrix{i,v}{t-1} \zsumnoise{v}{i}{t-1})}^2}
      \\\leq & (1+\beta)\underbrace{\sum_{i=1}^{n}\condexpect{t}{\norm{\sum_{v\in \neighbors{i}{t-1}}(\gossmatrix{i,v}{t-1}\model{v}{t-1}) - \avgmodel{t-1}- \sum_{v\in \neighbors{i}{t-1}}(\gossmatrix{i,v}{t-1} \lr \nabla\loss{v}(\model{v}{t-1},\datapoint{v}{t-1})) }^2}}_{:= T_1}
      \\ &+(1+\beta^{-1})\underbrace{\sum_{i=1}^{n}\condexpect{t}{\norm{\sum_{v\in \neighbors{i}{t-1}}(\gossmatrix{i,v}{t-1} \zsumnoise{v}{i}{t-1})}^2}}_{:= T_2}
    \end{align*}
  }

  Looking at the second term, and using~\cref{lem:norm_of_sum_reduction}:
  \simplifiedproof{%
    \begin{align*}
      T_2
      \leq & \sum_{i=1}^{n} \nodedegree{i} \sum_{v\in \neighbors{i}{t-1}} \condexpect{t}{\norm{\gossmatrix{i,v}{t-1} \zsumnoise{v}{i}{t-1}}^2}
      \\\leq & \sum_{i=1}^{n} \nodedegree{i} \sum_{v\in \neighbors{i}{t-1}} \condexpect{t}{(\gossmatrix{i,v}{t-1})^2\norm{\zsumnoise{v}{i}{t-1}}^2}
      \\\leq & \sum_{i=1}^{n} \nodedegree{i} \sum_{v\in \neighbors{i}{t-1}} \condexpect{t, i\in\neighbors{v}{t-1}}{(\gossmatrix{i,v}{t-1})^2\norm{\zsumnoise{v}{i}{t-1}}^2}
    \end{align*}%
  }{%
    \begin{align*}
      T_2
      \leq & \sum_{i=1}^{n} \nodedegree{i} \sum_{v\in \neighbors{i}{t-1}} \condexpect{t}{\norm{\gossmatrix{i,v}{t-1} \zsumnoise{v}{i}{t-1}}^2}
      \\\leq & \sum_{i=1}^{n} \nodedegree{i} \sum_{v\in \neighbors{i}{t-1}} \condexpect{t}{(\gossmatrix{i,v}{t-1})^2\norm{\zsumnoise{v}{i}{t-1}}^2}
      \\\leq & \sum_{i=1}^{n} \nodedegree{i} \sum_{v\in \neighbors{i}{t-1}} \condexpect{t, i\in\neighbors{v}{t-1}}{(\gossmatrix{i,v}{t-1})^2\norm{\zsumnoise{v}{i}{t-1}}^2}
      \\\leq& \sum_{i=1}^{n} \nodedegree{i} \sum_{v\in \neighbors{i}{t-1}} \condexpect{t, i\in\neighbors{v}{t-1}}{(\gossmatrix{i,v}{t-1})^2\condexpect{\gossmatrix{}{t-1}}{\norm{\zsumnoise{v}{i}{t-1}}^2}}
    \end{align*}%
  }
  Using~\cref{lem:corrected zsum variance} for a fixed gossip matrix, and leveraging $\gossmatrix{i,v}{t} = \gossmatrix{v,i}{t}$ since we assume symmetric matrices, we obtain:

  \singlecolumnproof{%
    \begin{align*}
      T_2
      \leq & \sum_{i=1}^{n} \nodedegree{i} \sum_{v=1}^{n} \condexpect{t, i\in\neighbors{v}{t-1}}{(\gossmatrix{v,i}{t-1})^2\modelsize\zsumvarmulti{v}{i}{t-1}}
      \\\leq & \modelsize\lr^2\sum_{i=1}^{n} \nodedegree{i} \sum_{v=1}^{n} \left(\frac{(\nodedegree{v} - 1) ^2}{\nodedegree{v}^2} +\frac{\nodedegree{v} - 1 }{\nodedegree{v}^2}\right)\ynoisevar{v}
      \\\leq&\modelsize\lr^2\sum_{i=1}^{n} \nodedegree{i} \sum_{v=1}^{n} \left(\frac{(\nodedegree{v} - 1) ^2}{\nodedegree{v}}\right)\ynoisevar{v}
    \end{align*}}{%
    \begin{align*}
      T_2
      \leq & \sum_{i=1}^{n} \nodedegree{i} \sum_{v=1}^{n} \condexpect{t, i\in\neighbors{v}{t-1}}{(\gossmatrix{v,i}{t-1})^2\modelsize\zsumvarmulti{v}{i}{t-1}}
      \\\leq & \modelsize\sum_{i=1}^{n} \nodedegree{i} \sum_{v=1}^{n} \condexpect{t, i\in\neighbors{v}{t-1}}{(\gossmatrix{i,v}{t-1})^2\left(\frac{(\nodedegree{v} - 1) ^2}{\nodedegree{v}^2} +\frac{\sum_{j\in\neighbors{v}{t}, j\neq v}(\gossmatrix{v,j}{t})^2 }{(\nodedegree{v}\gossmatrix{v,i}{t})^2}\right)\lr^2\ynoisevar{v}}
      \\\leq & \modelsize\lr^2\sum_{i=1}^{n} \nodedegree{i} \sum_{v=1}^{n} \condexpect{t, i\in\neighbors{v}{t-1}}{\left(\frac{(\nodedegree{v} - 1) ^2(\gossmatrix{i,v}{t-1})^2}{\nodedegree{v}^2} +\frac{\sum_{j\in\neighbors{v}{t}, j\neq v}(\gossmatrix{j,v}{t})^2 }{\nodedegree{v}^2}\right)\ynoisevar{v}}
      \\\leq & \modelsize\lr^2\sum_{i=1}^{n} \nodedegree{i} \sum_{v=1}^{n} \left(\frac{(\nodedegree{v} - 1) ^2}{\nodedegree{v}^2} +\frac{\nodedegree{v} - 1 }{\nodedegree{v}^2}\right)\ynoisevar{v}
      \\\leq&\modelsize\lr^2\sum_{i=1}^{n} \nodedegree{i} \sum_{v=1}^{n} \left(\frac{(\nodedegree{v} - 1) ^2}{\nodedegree{v}}\right)\ynoisevar{v}
    \end{align*}
  }
  Where we used that $(\gossmatrix{i,v}{})^2\leq 1 $ for all $i,v\in\nodeset$.

  \simplifiedproof{%
    We obtain that $T_1$ is equal to the expectation $\condexpect{t}{}$ of:
    \begin{align*}
      \norm{\gossmatrix{}{t-1}\left(\model{}{t-1} -  \lr \nabla\loss{}(\model{}{t-1},\datapoint{}{t-1})\right) - \avgmodel{t-1}}_F^2
    \end{align*}%
  }{%
    For $T_1$, we obtain that:
    \begin{align*}
      T_1
      = & \condexpect{t}{\norm{\gossmatrix{}{t-1}\left(\model{}{t-1} -  \lr \nabla\loss{}(\model{}{t-1},\datapoint{}{t-1})\right) - \avgmodel{t-1}}_F^2}
    \end{align*}%
  }
  This is the exact notation from~\cite{koloskovaUnifiedTheoryDecentralized2020}, in the proof of the corresponding Lemma (Lemma 9), with the notation $\tau = 1$ (our matrix notation are transposed to theirs).
  By following the same steps, we obtain:
  \singlecolumnproof{
    \begin{align*}
      T_1
      \leq & n(1-\frac{p}{2})\consdistance{t-1}
      + n\frac{p}{16} \consdistance{t-1}
      + n(\gradientnoiseavg+ \frac{18}{p}\gradientnormavg) \lr^2
      \\+ & n\frac{36L}{p}\lr^2 (\localloss{}(\avgmodel{t-1}) - \localloss{}(\optimum))
    \end{align*}
  }{%
    \begin{align*}
      T_1
      \leq & n\left((1-\frac{p}{2})\consdistance{t-1}
      + \frac{p}{16} \consdistance{t-1}
      + \frac{36L}{p}\lr^2 (\localloss{}(\avgmodel{t-1}) - \localloss{}(\optimum))
      + (\gradientnoiseavg+ \frac{18}{p}\gradientnormavg) \lr^2\right)
    \end{align*}
  }

  \simplifiedproof{Plugging $T_1$ and $T_2$ back into the original term, we obtain the desired result.
  }{%
  \singlecolumnproof{
    Plugging $T_1$ and $T_2$ back into the original term, we obtain:
    \begin{align*}
      \consdistance{t}
      \leq & (1+\beta)\left((1-\frac{7p}{16})\consdistance{t-1}
      + \frac{36L}{p}\lr^2 (\localloss{}(\avgmodel{t-1}) - \localloss{}(\optimum))\right)
      \\ & + (1+\beta)\left((\gradientnoiseavg+ \frac{18}{p}\gradientnormavg) \lr^2\right)
      \\ & + (1 + \beta^{-1})\frac{\modelsize\lr^2}{n} \sum_{i=1}^{n} \nodedegree{i} \sum_{v=1}^{n} \left(\frac{{(\nodedegree{v} - 1)}^2}{\nodedegree{v}}\right)\ynoisevar{v}
      \\\leq & (1+\beta)(1-\frac{7p}{16})\consdistance{t-1}
      +  (1+\beta)\frac{36L}{p}\lr^2 (\localloss{}(\avgmodel{t-1}) - \localloss{}(\optimum))
      \\ & +\lr^2 (1+\beta)(\gradientnoiseavg+ \frac{18}{p}\gradientnormavg)
      \\ & + \lr^2 \left((1 + \beta^{-1})\frac{\modelsize}{n} \sum_{i=1}^{n} \nodedegree{i} \sum_{v=1}^{n} \left(\frac{{(\nodedegree{v} - 1)}^2}{\nodedegree{v}}\ynoisevar{v}\right)\right)
    \end{align*}
    For any $\beta>0$, which is the desired result.
  }{
    Plugging $T_1$ and $T_2$ back into the original term, we obtain:
    \begin{align*}
      \consdistance{t}
      \leq & (1+\beta)\left((1-\frac{7p}{16})\consdistance{t-1}
      + \frac{36L}{p}\lr^2 (\localloss{}(\avgmodel{t-1}) - \localloss{}(\optimum))
      + (\gradientnoiseavg+ \frac{18}{p}\gradientnormavg) \lr^2\right)
      \\ & + (1 + \beta^{-1})\frac{\modelsize\lr^2}{n} \sum_{i=1}^{n} \nodedegree{i} \sum_{v=1}^{n} \left(\frac{{(\nodedegree{v} - 1)}^2}{\nodedegree{v}}\right)\ynoisevar{v}
      \\\leq & (1+\beta)(1-\frac{7p}{16})\consdistance{t-1}
      +  (1+\beta)\frac{36L}{p}\lr^2 (\localloss{}(\avgmodel{t-1}) - \localloss{}(\optimum))
      \\ & + \left((1+\beta)(\gradientnoiseavg+ \frac{18}{p}\gradientnormavg)
      + (1 + \beta^{-1})\frac{\modelsize}{n} \sum_{i=1}^{n} \nodedegree{i} \sum_{v=1}^{n} \left(\frac{{(\nodedegree{v} - 1)}^2}{\nodedegree{v}}\ynoisevar{v}\right)\right)\lr^2
    \end{align*}
    For any $\beta>0$, which is the desired result.
  }
  }
\end{proof}
 
\end{document}